\documentclass{article} % For LaTeX2e
\usepackage{iclr2027_conference,times}

\usepackage{amsmath,amsfonts,bm}

\def\eqref#1{equation~\ref{#1}}
\def\Eqref#1{Equation~\ref{#1}}
\def\1{\bm{1}}

\DeclareMathAlphabet{\mathsfit}{\encodingdefault}{\sfdefault}{m}{sl}
\SetMathAlphabet{\mathsfit}{bold}{\encodingdefault}{\sfdefault}{bx}{n}

\newcommand{\R}{\mathbb{R}}

\newcommand{\KL}{D_{\mathrm{KL}}}

\usepackage{enumitem}
\usepackage{fancyvrb}
\usepackage[table]{xcolor}
\usepackage{booktabs}
\usepackage{graphicx}
\usepackage{float}

\usepackage{acronym}
\acrodef{SSL}[SSL]{self-supervised learning}
\acrodef{AI}[AI]{artificial intelligence}
\acrodef{DM}[LDM]{latent distribution matching}
\acrodef{MI}[MI]{mutual information}
\acrodef{ICA}[ICA]{independent component analysis}
\acrodef{SFA}[SFA]{slow feature analysis}
\acrodef{stopgrad}[stopgrad]{stop gradient}
\acrodef{JEPA}[JEPA]{joint-embeddding predictive architecture}
\acrodef{CPC}[CPC]{contrastive predictive coding}
\acrodef{MLP}[MLP]{multi-layer perceptron}
\acrodef{LSTM}[LSTM]{long short-term memory}
\acrodef{infoLDM}[InfoLDM]{information maximizing latent distribution matching}
\usepackage{microtype}
\usepackage{amsmath,amssymb,amsthm,mathtools}
\newcommand\identity{1\kern-0.25em\text{l}}
\DeclareUnicodeCharacter{03A9}{\ensuremath{\Omega}}
\DeclareUnicodeCharacter{03B9}{\ensuremath{\iota}}
\DeclareUnicodeCharacter{03BA}{\ensuremath{\kappa}}
\DeclareUnicodeCharacter{03C9}{\ensuremath{\omega}}
\DeclareUnicodeCharacter{1D50}{\ensuremath{{}^{\mathrm m}}}
\DeclareUnicodeCharacter{1D62}{\ensuremath{{}_{\mathrm i}}}
\DeclareUnicodeCharacter{2080}{\ensuremath{{}_{0}}}
\DeclareUnicodeCharacter{211D}{\ensuremath{\mathbb{R}}}
\DeclareUnicodeCharacter{21A6}{\ensuremath{\mapsto}}
\DeclareUnicodeCharacter{2200}{\ensuremath{\forall}}
\DeclareUnicodeCharacter{2202}{\ensuremath{\partial}}
\DeclareUnicodeCharacter{2203}{\ensuremath{\exists}}
\DeclareUnicodeCharacter{2208}{\ensuremath{\in}}
\DeclareUnicodeCharacter{2227}{\ensuremath{\wedge}}
\DeclareUnicodeCharacter{2264}{\ensuremath{\leq}}
\DeclareUnicodeCharacter{22A4}{\ensuremath{\top}}
\DeclareUnicodeCharacter{27C2}{\ensuremath{\perp}}
\usepackage[hidelinks]{hyperref}
\usepackage{fontawesome5}
\usepackage{url}
\usepackage{tikz}
\usetikzlibrary{arrows.meta,backgrounds,calc,fit,positioning}

\definecolor{linkblue}{HTML}{2563A6}

\newcommand{\Normal}{\mathcal{N}}
\newcommand{\rank}{\operatorname{rank}}
\newcommand{\im}{\operatorname{im}}
\newcommand{\indep}{\mathrel{\perp\!\!\!\perp}}

\newcommand{\cslthreefactorpanel}[1]{%
    \begin{tikzpicture}[
        baseline=(current bounding box.center),
        x=8.4mm,
        factor/.style={
            circle,
            draw=black!55,
            semithick,
            minimum size=7.5mm,
            inner sep=0.5pt,
            align=center,
            font=\scriptsize
        },
        signal/.style={-{Latex[length=1.4mm]},black,semithick}
    ]
        \foreach \row/\y in {1/0,2/-1.0} {
            \node[font=\scriptsize\bfseries,anchor=e] at (-0.48,\y) {Img.~\row};
            \node[factor] (ob\row) at (0,\y) {$i_{\mathrm{obj}}$};
            \node[factor] (p\row) at (1,\y) {$p_{\mathrm{obj}}$};
            \node[factor] (theta\row) at (2,\y) {$\theta_{\mathrm{obj}}$};
            \node[factor] (ho\row) at (3,\y) {$h_{\mathrm{obj}}$};
            \node[factor] (ps\row) at (4,\y) {$p_{\mathrm{spot}}$};
            \node[factor] (hs\row) at (5,\y) {$h_{\mathrm{spot}}$};
            \node[factor] (hb\row) at (6,\y) {$h_{\mathrm{back}}$};
        }
        \foreach \factor in {#1} {
            \draw[signal] (\factor1.south) -- (\factor2.north);
        }
    \end{tikzpicture}%
}

\newtheorem{assumption}{Assumption}
\newtheorem{theorem}{Theorem}
\newtheorem{corollary}{Corollary}

\theoremstyle{remark}

\title{Predictive Self-Supervised Learning Provably Identifies Stochastic Signals under Nuisance}
\author{Fabian A. Mikulasch \textsuperscript{1}\& Friedemann Zenke \textsuperscript{1,2} \\
\textsuperscript{1}Friedrich Miescher Institute for Biomedical Research, Basel, Switzerland \\ 
\textsuperscript{2}Faculty of Science, University of Basel, 
Switzerland\\ 
\texttt{\{firstname.lastname\}@fmi.ch}}

\iclrpreprintcopy % Show authors without an ICLR publication or review header.
\begin{document}

\maketitle

\begin{abstract}
\Ac{SSL} by predicting in latent space, without generating the input data itself, learns highly abstract, useful representations.
Intuitively, this success is often attributed to its ability to discard nuisance information that is irrelevant to prediction.
% However, it is not clear if these models can distinguish true nuisance from simple stochastic variation in the prediction-relevant latent signal.
% Existing identifiability results fall short here: one line of work allows the latent signal to evolve stochastically but admits no nuisance, while another permits nuisance but requires the signal to be shared deterministically across observations. 
% Existing identifiability results under nuisance fall short here, since they require the signal to be shared deterministically across observations.
% We close this gap by introducing a latent-variable model with stochastic latent dynamics and unpredictable nuisance variables, and show that common \ac{SSL} methods implicitly instantiate it. 
% However, this poses a conundrum: both stochastic variation in a prediction-relevant latent signal and true nuisance make observations unpredictable; how could they be distinguished?
However, this poses a conundrum: both stochastic variation in a prediction-relevant latent signal and true nuisance make observations partly unpredictable; how could they be distinguished?
% Surprisingly, common \ac{SSL} methods can achieve exactly this, by implicitly instantiating a latent-variable model with stochastic dynamics and private nuisance.
Surprisingly, we prove that common \ac{SSL} methods can achieve exactly this, by implicitly instantiating a latent-variable model with stochastic dynamics and observation-private nuisance.
We trace their ability to recover the stochastic signal to two complementary principles:
Predictive \acl{MI} maximization ensures that representations retain the information needed for prediction, while
\acl{DM} constrains how this information is encoded, thereby making the retained signal identifiable.
We confirm this identifiability result in simulations for Gaussian predictors, which recover the true signal up to an affine transformation even in dynamic, nuisance-laden environments. 
% \atright{\small{\href{https://info-ldm.github.io/}{\textcolor{linkblue}{\faGlobe\ Project Website}}}}
\end{abstract}

\section{Introduction}

\Acf{SSL} has become a central paradigm for representation learning, enabling models to extract useful structure from unlabeled data across vision, language, and audio \citep{noroozi_unsupervised_2016,oord_representation_2019,dawid2024introduction,gui2024survey}.
Its success is often attributed to a simple intuition: different observations in the same context share a context-specific signal while also containing observation-specific variation that should not be represented.
For example, image augmentations alter image details without changing object identity, different modalities can provide different views of the same signal, and temporal observations combine an evolving signal with non-predictive nuisance.
By learning what is shared or predictable across observations, non-generative modeling is thought to retain the underlying \emph{signal} while \emph{nuisance} variables are not represented but left implicit in the generative model.

This nuisance-based explanation is pervasive, but it is rarely made explicit in statistical models.
% For example, non-generative modeling has been explained through a geometric view, which describes learning as alignment of paired representations combined with uniformity or variance constraints that prevent collapse \citep{wang2020understanding}.
For example, non-generative modeling has been derived from the idea that representations should preserve the \ac{MI} shared between observations, while unpredictable pixel-level nuisance might be ignored \citep{oord_representation_2019,tian2020contrastive,shwartz2023information,galvez2023role,shwartz_ziv_compress_2024}.
% However alone does not specify which variations constitute nuisance or guarantee that the remaining coordinates become meaningful representations of the signal.
Although these information-theoretic approaches provide an intuitive view of how nuisance is treated in non-generative models, they do not formally specify what exactly constitutes nuisance, and how the remaining signal is represented.

One route from the intuitive view to formal insights is given by identifiability theory.
Results in nonlinear \ac{ICA} and non-generative modeling show that latent variables can be recovered up to restricted transformations when the learned and true conditional distributions have suitable structure \citep{hyvarinen2019nonlinear}.
This analysis has been extended to prove that non-generative models can identify the signal even under the presence of nuisance \citep{kugelgen_self-supervised_2022}.
These results require that the model is provided with different views of the same data point, where nuisance is private to each view, and the signal variable is \emph{exactly} the same for each view.
However, this does not cover the more general setting where the signal can evolve stochastically itself.
This raises the question: is it possible to distinguish noise in otherwise predictable dynamics from completely unpredictable nuisance?

Here we formalize implicit nuisance in stochastic latent variable models and study when non-generative modeling recovers the nuisance-free signal.
% More precisely, we show that this recovery can be enabled by two principles:
% Predictive \ac{MI} maximization, which forces representations to retain sufficient information to predict future observations \citep{oord_representation_2019}, and \ac{DM}, which requires these latent representations to follow a simple distribution, and enables identification of the retained signal \citep{zimmermann_contrastive_2022,mikulasch2026understanding}.
Our main contributions are:

\begin{description}[style=unboxed,leftmargin=0cm,nolistsep]
\setlength\itemsep{4pt}
\item[Disentangling the roles of \ac{MI} maximization and distribution matching.] We show that in non-generative SSL, predictive \ac{MI} maximization and \ac{DM} play distinct roles: \ac{MI} maximization guarantees that representations retain all predictable information, while \ac{DM} alone is responsible for identification of the retained signal. This resolves prior ambiguity about the function of the terms in non-generative SSL objectives \citep{tschannen2019mutual,mikulasch2026understanding}, showing that \ac{MI} maximization and \ac{DM} are not sufficient on their own, but complement each other.

\item[Identifiability under stochastic signals and private nuisance.]
We introduce a latent-variable model in which nuisance is private to each observation given a stochastically evolving signal, relaxing the deterministic shared-context assumption of prior identifiability results in self-supervised learning \citep{kugelgen_self-supervised_2022,daunhawer_identifiability_2023}. 
For exponential-family predictive models we prove that exact \ac{DM} together with predictive-information saturation recovers the true sufficient statistic as an affine readout of the learned statistic.

\item[Affine recovery for Gaussian predictors.]
We specialize our general result to the practically relevant case of Gaussian predictive models, showing that the signal is recovered as an affine transformation of the learned representation whenever the latent dimensionality is large enough. 
We confirm this in simulations with stochastic latent dynamics and unpredictable nuisance, and in a stochastic variant of the Causal3DIdent benchmark, demonstrating recovery of causally coupled signal factors from high-dimensional images. % Mojoco action

\item[Comparing different entropy estimators for system identification.]
Our results show that identification performance depends on the entropy estimator used to compute the objective function and prevent representational collapse.
Our main finding is that score-based entropy maximization, which does not explicitly estimate the empirical latent distribution, provides the best robustness and performance.
\enlargethispage{3mm}

\end{description}

\section{Related work}

% \begin{description}[style=unboxed,leftmargin=0cm,nolistsep]
% \setlength\itemsep{4pt}

\textbf{Theory of non-generative latent models.}
There has been significant interest in understanding non-generative modeling through theoretical analysis \citep[e.g.,][]{Arora_Khandeparkar_Khodak_Plevrakis_Saunshi_2019,wang2020understanding,Ben-Shaul_Shwartz-Ziv_Galanti_Dekel_LeCun_2023}.
A line of work closely related to this article aims to connect contrastive methods (e.g., CPC) to latent variable models \citep{kugelgen_self-supervised_2022,zimmermann_contrastive_2022,aitchison_infonce_2023,Nakamura_Okada_Taniguchi_2023, bizeul_probabilistic_2024,mikulasch2026understanding}.
Others proposed information-theoretic interpretations of regularization-based methods \citep[e.g., VICReg, ][]{shwartz2023information} and clustering/contrastive methods \citep{galvez2023role}.

\textbf{Identifiable latent variable models without nuisance.}
Identifiability is a core concept in linear \ac{ICA} \cite{hyvarinen1999independent}, which guarantees the recovery of "true" underlying variables up to trivial transformations.
The concept was later extended to nonlinear \ac{ICA} using assumptions such as temporal structure or auxiliary variables \citep{sprekeler2014extension,khemakhem2020variational,hyvarinen2019nonlinear,roeder2021linear}. 
More recently, these insights also enabled proving identifiability for non-generative models \citep{zimmermann_contrastive_2022,laiz2024self,mikulasch2026understanding}.
While these models study identification under stochastic dynamics, they do not include nuisance in their generative model.
One notable exception is GIN \citep{sorrenson2020disentanglement}, which can distinguish between nuisance and stochastic signal variables, but relies on invertible encoders.

\textbf{Identifiable latent variable models with nuisance.}
A related line of work studies identifiability in the presence of view-specific nuisances through decoder-free models that can avoid modeling this nuisance. 
\citet{gresele2020incomplete} showed that shared sources can be identified from multiple nonlinear views despite view-specific corruptions. 
\citet{kugelgen_self-supervised_2022} and \citet{lyu2021understanding} extended this approach to show that common \ac{SSL} methods might be used to extract the shared components between views. 
\citet{daunhawer_identifiability_2023} transported these insights to multimodal settings, proving recovery of factors shared across modalities in the presence of modality-specific latent variation. 
These works establish that identifiability is possible when nuisance variables are present, but require deterministic relationships of the recovered variables between views.

% \end{description}

\section{Theory}

% \section{Predictive information maximization and latent distribution matching}
Before we develop our theory of predictive \ac{SSL} under nuisance we briefly review the concepts of predictive \ac{MI} maximization and latent distribution matching that lie at the core of our work.

The specific loss function we are going to analyze has been used in several previous works.
Initially, \citet{oord_representation_2019} proposed InfoNCE, to maximize predictive MI $I[z;z_c]$ between representations $z$ and predictive variables $z_c$. 
% A problem of this idea is that \ac{MI} is invariant under arbitrary invertible reparametrizations and thus does not by itself give rise to useful representations \citep{tschannen2019mutual}.
Later work has emphasized that the objective that is actually maximized by many of the common \ac{SSL} algorithms, including InfoNCE \citep{oord_representation_2019}, SimCLR \citep{chen2020simple}, and VICReg \citep{shwartz2023information}, has the general form
\citep{aitchison_infonce_2023,mikulasch2026understanding}
\begin{align}
\begin{split}
    \mathcal F(f,f_c,\theta)
    =&
    -\KL[q_f(z,z_c)\|p_\theta(z,z_c)]+I_{q_f}[z;z_c]\\
    =&\langle\log p_\theta(z, z_c) \rangle_{q_f(z, z_c)} + H_{q_f}[z]+ H_{q_f}[z_c] \; .
    \label{eq:infoLDM}
\end{split}
\end{align}
Here $q_f(z,z_c)$ is the pushforward distribution of observations $p(x,x_c)$ through the observation encoder $z=f(x)$ and condition encoder $z_c=f_c(x_c)$, and $p_\theta(z, z_c)$ is a latent model that is learned alongside.
Different entropy estimators lead to different \ac{SSL} algorithms---for example, KDE and log-determinant of covariance estimators can be related to SimCLR \citep{chen2020simple} and VICReg \citep{shwartz2023information}, respectively \citep{mikulasch2026understanding}.
By this goal function, representations are not only required to be mutually informative, they are also required to follow a simple learned latent model $p_\theta(z, z_c)$ via \ac{DM}.
In the following we will therefore refer to \Eqref{eq:infoLDM} as the InfoLDM goal function.

Commonly, but not necessarily, this goal function is simplified by ignoring the marginal of the conditioning variable and setting $p_\theta(z_c)=q_f(z_c)$, leading to
\begin{align}
    \mathcal F(f,f_c,\theta)
    =&\langle\log p_\theta(z\mid z_c) \rangle_{q_f(z, z_c)} + H_{q_f}[z] \; .
    \label{eq:infoLDMsimple}
\end{align}
This is a well-known lower bound on the \ac{MI} between latent variables \citep{poole2019variational}, but we here analyze it with the clear framing that maximizing it performs both \ac{MI} maximization and \ac{DM}.

% Previously, \citet{mikulasch2026understanding} analyzed the goal function (\ref{eq:infoLDM}) and found that \ac{DM} via the KL term is essential for finding identifiable representations.
% This analysis, however, did not cover the role of the \ac{MI} term, and assumed that $f$ is maintained to be invertible on the data manifold.
% Here we extend this analysis to the more general setting where $f$ is not necessarily invertible, and show that both terms play crucial roles: The KL term simplifies representations for identification and the \ac{MI} term leads representations to maintain predictive information, while nuisance information may be discarded.

% \subsection{Identifiability under implicit nuisance}

\subsection{Exponential family statistical recovery}

The main theorem specifies when maximizing the InfoLDM goal function (\Eqref{eq:infoLDM}) recovers the signal $s$ in the learned representation $z$ up to simple transformations. 
For this we make the following assumptions, which are specified more rigorously in Appendix~\ref{app:assumptions}.

\textbf{Generative model.}
% All random variables below take values in standard Borel spaces, all displayed deterministic maps are measurable, and all displayed conditional distributions are Markov kernels.
Let $s\in\mathcal S\subseteq\R^{d_s}$ be a semantic \emph{signal} variable, $n\in\R^{d_n}$ a noisy \emph{nuisance} variable, and let $c$ denote the true \emph{condition} (or cause, context).
Let $x=g(s,n)$ be the observation, and $x_c$ the observed condition.
We understand the observed condition $x_c$ quite generally as combining variables that allow to predict the signal, e.g., including previous observations or actions.
As before, let $z=f(x)\in\mathcal Z\subseteq\R^{d_z}$ be the learned representation, and $z_c=f_c(x_c)\in\mathcal Z_c\subseteq\R^{d_{z_c}}$ the learned encoded condition.
Finally, let $q_f(z,z_c)$ be the pushforward distribution of observations, and $p_\theta(z, z_c)$ the learned model (Figure~\ref{fig:private-nuisance-dag}).
% The variable $x_c$ is included to distinguish the true, observed, and encoded conditions.
% The conditional-transfer argument only needs the combined channel $p(z_c\mid c)$, while the separate channel $p(x_c\mid c)$ determines whether the observed condition makes all predictively relevant information available to the learned condition encoder.

\begin{figure}[t]
    \centering
    \begin{tikzpicture}[
        >={Latex[length=2.2mm]},
        node distance=5mm and 14mm,
        random/.style={circle,draw,minimum size=8mm,inner sep=0pt},
        deterministic/.style={rectangle,draw,minimum size=8mm,inner sep=0pt},
        every edge/.style={draw,->,semithick}
    ]
        \node[random] (c) {$c$};
        \node[random,below=of c] (s) {$s$};
        \node[random,below=of s] (n) {$n$};
        \node[random,right=of c] (xu) {$x_c$};
        \node[deterministic,right=of s] (x) {$x$};
        \node[deterministic,right=of xu] (zu) {$z_c$};
        \node[deterministic,right=of x] (z) {$z$};
        \node[text=blue!60!black,font=\fontsize{8}{8}\selectfont,align=center,left=5mm of c] (cond) {Conditioning variables\\(e.g., history, actions, context)};
        \node[text=blue!60!black,font=\fontsize{8}{8}\selectfont,align=center] (sig) at (cond |- s) {Signal variables};
        \node[text=blue!60!black,font=\fontsize{8}{8}\selectfont,align=center] (nuis) at (cond |- n) {Nuisance variables};
        \node[text=orange!80!black,font=\fontsize{8}{8}\selectfont,align=center,right=12mm of zu] (cond2) {Encoded condition};
        \coordinate (sn-midpoint) at ($(s)!0.5!(n)$);
        \node[text=orange!80!black,font=\fontsize{8}{8}\selectfont,align=center] (sn2) at (cond2 |- z) {Encoded observation};

        \begin{scope}[on background layer]
            \node[draw=blue!60!black,dotted,rounded corners=4pt,thick,
                inner sep=3mm,fit=(c) (s) (n) (xu) (x),
                label={[text=blue!60!black,font=\bfseries,above]Generative model}] {};
            \node[draw=orange!80!black,dotted,rounded corners=4pt,thick,
                inner xsep=10mm,inner ysep=3mm,fit=(zu) (z),
                label={[text=orange!80!black,font=\bfseries,above]Learned model}] {};
        \end{scope}

        \path
            (c) edge (s)
            (c) edge (xu)
            (s) edge (n)
            (s) edge[blue!50!black]  (x)
            (n) edge[blue!50!black] node[midway,above=2mm,text=blue!50!black]{$g$} (x)
            (xu) edge[green!50!black] node[midway,above,text=green!50!black]{$f_c$} (zu)
            (zu) edge[bend left=60,draw=red!60!black] node[midway,right,text=red!60!black] {$p_\theta$} (z)
            (x) edge[green!50!black] node[midway,above,text=green!50!black] {$f$} (z);

    \end{tikzpicture}
    \caption{Graphical model for observation-private nuisance.
    The three layers show (from left to right) unobserved latent variables, observations, and learned representations.
    Square nodes denote deterministically defined variables.
    % The nuisance may depend on the signal, but it has no additional dependence on the condition, i.e., $n\indep c\mid s$.
    % Both $s$ and $n$ generate the current observation $x$, from which the current representation $z$ is encoded.
    }
    \label{fig:private-nuisance-dag}
\end{figure}
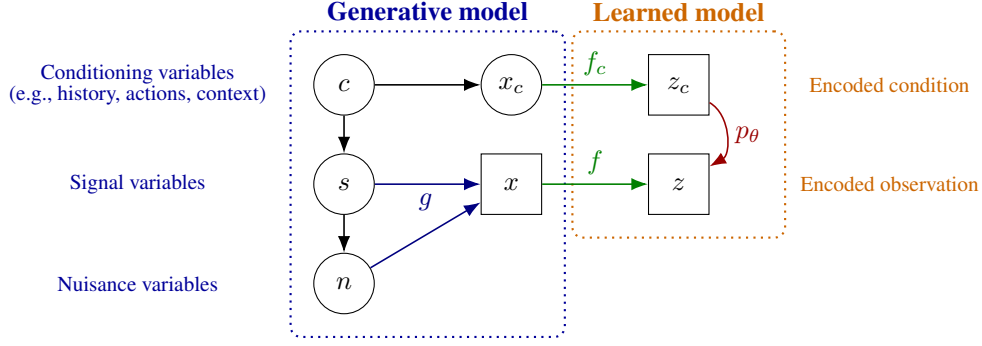

% \paragraph{Assumptions}

%\begin{enumerate}

\textbf{Assumption 1.} The nuisance is private to the current observation in the sense that
\begin{equation}
    n\indep c\mid s.
\end{equation}

\textbf{Assumption 2.} The true predictive distribution is exponential family
\begin{equation}
    p(s\mid c)
    =h_s(s)
    \exp\!\left(
        \alpha(c)^\top\tau_\star(s)-\psi(\alpha(c))
    \right),
\end{equation} 
where $\tau_\star:\mathcal S\to\R^{r_s}$ is a sufficient statistic, $h_s$ a fixed carrier function, and $\alpha(c)$ encodes the dependency on the condition $c$. Similarly, the learned latent model is exponential family
\begin{equation}
    p_\theta(z\mid z_c)
    =h_z(z)
    \exp\!\left(
        \beta(z_c)^\top\tau_z(z)-\phi(\beta(z_c))
    \right),
\end{equation}
where the predictor $\beta(z_c)$ is learned, and $\tau_z:\mathcal Z\to\R^{r_z}$.

\textbf{Assumption 3.} By optimizing the InfoLDM goal function Equation~\ref{eq:infoLDM} via $q_f(z,z_c)$ and $p_\theta(z,z_c)$ it is possible to reach the global optimum where
\begin{equation}
    \KL[q_f(z,z_c)\|p_\theta(z,z_c)]=0,
    \qquad
    I_{q_f}[z;z_c]=I[(s,n);c]=I[s;c]<\infty.
\end{equation}

\textbf{Assumption 4.} It is possible to choose $c_0,c_1,\ldots,c_{r_s}$ from the support of $p(c)$ such that
\begin{equation}
    D
    :=
    \begin{bmatrix}
        (\alpha(c_1)-\alpha(c_0))^\top\\
        \vdots\\
        (\alpha(c_{r_s})-\alpha(c_0))^\top
    \end{bmatrix}
    \in\R^{r_s\times r_s}
\end{equation}
is invertible.

\textbf{Assumption 5.} For every condition $c$ the random vector $\tau_\star(s)$ with $s\sim p(s\mid c)$ is not contained in a proper affine hyperplane of $\R^{r_s}$.
%\end{enumerate}

\medskip

Most of the presented assumptions are standard in the context of nonlinear identification theory.
Assumption~\ref{ass:markov} can be understood as a definition of nuisance factors, similar to the definition by \citet{kugelgen_self-supervised_2022}.
Assumptions~\ref{ass:exponential-families},~\ref{ass:condition-diversity} and~\ref{ass:source-nondegenerate} are standard assumptions in the theory of exponential family identification \citep[e.g.,][]{khemakhem2020variational}.
Chiefly,~\ref{ass:condition-diversity} and~\ref{ass:source-nondegenerate} ensure that all latent dimensions of the probabilistic model, not only a subset, are well constrained.

Assumption~\ref{ass:ideal-optimum} is the strongest. It implicitly replaces common invertibility assumptions on the generative model.
To see this relation consider that the generative model (Figure~\ref{fig:private-nuisance-dag}) together with the data processing inequality imply $
    I_{q_f}[z;z_c]
    \le I[x;x_c]
    \le I[s;c]
    =I[(s,n);c].
$
At the optimum these inequalities become equalities, and in particular information has to be preserved in the observations $I[x;x_c] = I[s;c]$, restricting $g(s,n)$ and $p(x_c\mid c)$ to be information-preserving.
Additionally, this restriction implicitly requires that the latent representation is sufficiently high-dimensional to capture the required information, as will be demonstrated by Theorem~\ref{thm:exponential-readout}.

% \subsection{Main result: exponential-family statistic recovery}
\medskip % if you need a break 
With these assumptions in place we are ready to state our main result.

\begin{theorem}[Exponential-family statistical recovery through InfoLDM]
\label{thm:exponential-readout}
Suppose Assumptions~\ref{ass:markov}, \ref{ass:exponential-families}, \ref{ass:ideal-optimum}, \ref{ass:condition-diversity}, and~\ref{ass:source-nondegenerate} hold, and the InfoLDM goal function (\Eqref{eq:infoLDM}) is fully maximized.

Then the dimension of the model sufficient statistic $r_z\ge r_s$, and there is a full-row-rank matrix $A\in\R^{r_s\times r_z}$ and a vector $b\in\R^{r_s}$ such that
\begin{equation}
    \boxed{\tau_\star(s)=A\tau_z(z)+b}
    \qquad\text{almost surely}.
    \label{eq:main-readout}
\end{equation}
If $\tau_\star$ is injective, then $s$ is almost surely a measurable function of $z$.
If $r_z=r_s$, then $A$ is invertible and the two sufficient statistics are related by an invertible affine transformation.
\end{theorem}

The full proof is given in Appendix~\ref{app:proof}.
In broad strokes, the first step is to use successful MI maximization to show that $z_c$ retains all information about $s$ from $c$, and similarly $z$ about $c$.
This can be leveraged to directly relate the true and learned conditional probability distributions of $s$ and $z$. 
From there the result follows via well established identifiability theory approaches for exponential family models \citep[e.g.,][]{khemakhem2020variational}.
% The inequality $r_z\ge r_s$ is a conclusion rather than an assumed relation between dimensions: source nonredundancy forces the affine readout $A:\R^{r_z}\to\R^{r_s}$ to be surjective.
% It concerns the dimensions of the sufficient statistics, not the raw dimensions $d_s$ and $d_z$ of the signal and learned code.

\subsection{Specialization to Gaussian prediction}

Theorem~\ref{thm:exponential-readout} is fairly general; here we aim to understand the behavior of commonly used non-generative models, which rely on Gaussianity assumptions.
The goal function of many of these can be understood as approximately maximizing
\begin{align}
    \label{eq:gaussloss}
    \mathcal{F} \propto \left\langle - \frac{1}{2\sigma^2}  \parallel z - \mu_\theta(z_c) \parallel^2 \right\rangle_{q_f(z,z_c)} + H_{q_f}[z] ,
\end{align}
where $\mu_\theta$ is a learned prediction function (cf.\ \citealp{kugelgen_self-supervised_2022,mikulasch2026understanding}).
This expression follows directly from the simplified goal function (\Eqref{eq:infoLDMsimple}) by choosing $p_\theta$ to be a Gaussian with fixed variance $\Sigma_\theta=\sigma^2 \identity$ \citep{shwartz2023information,mikulasch2026understanding}.
Further asserting
\begin{equation}
    p(s\mid c)=\Normal(s;\mu_\star(c),\Sigma_\star) ,
    \qquad
    p_\theta(z\mid z_c)=\Normal(z;\mu_\theta(z_c),\Sigma_\theta) ,
    \qquad
    \Sigma_\star\succ0 ,
    \quad
    \Sigma_\theta\succ0 ,
    \label{eq:gaussian-conditionals}
\end{equation}
Theorem~\ref{thm:exponential-readout} can be specialized to give the stronger result of affine recovery of $s$.

\begin{corollary}[Gaussian affine recovery]
\label{cor:gaussian-readout}
Under the assumptions of Theorem~\ref{thm:exponential-readout} and \Eqref{eq:gaussian-conditionals}, necessarily $d_z\ge d_s$, and there is a full-row-rank matrix $A\in\R^{d_s\times d_z}$ and $b\in\R^{d_s}$ such that
\begin{equation}
    s=Az+b
    \qquad\text{almost surely}.
    \label{eq:gaussian-readout}
\end{equation}
If $d_z=d_s$, then $A$ is invertible.
\end{corollary}

This result follows immediately from Theorem~\ref{thm:exponential-readout} by the fact that the statistic of the fixed variance Gaussian distribution is the identity $\tau_\star(s)=s$ and $\tau_z(z)=z$, and $d_s=r_s$, $d_z=r_z$. 
Similar results might be obtained for other exponential family distributions, such as von Mises--Fisher for variables on the sphere.
% The true natural parameter is $\alpha(c)=\Sigma_\star^{-1}\mu_\star(c)$, so if the predictive means vary over a full-dimensional region this implies condition diversity (Assumption~\ref{ass:condition-diversity}).

Intuitively, affine recovery fully relies on noise in the dynamics, as described in previous analyses \citep{zimmermann_contrastive_2022,laiz2024self,mikulasch2026understanding}: Noise in the true latent variables has Gaussian structure, and forcing learned variables to be Gaussian locally straightens latent coordinates around the predicted mean, which globally leads to an affine relation between true and recovered variables. 
Information sufficiency on the other hand supplies the model with a guarantee that all predictable signal variables are recovered.
Note that nothing forces the model to discard nuisance variables entirely, except if the latent dimensionality of the model is restricted to match the dimensionality of the true signal.
However, it can be shown that if the representation is overcomplete, signal information is contained within a $d_s$-dimensional linear subspace, while the rest can only contain nuisance information (Appendix, Corollary~\ref{cor:gaussian-nuisance}, Table \ref{tab:randstep-latent-dimension}).

In the Appendix we provide additional theoretical results for these Gaussian dynamics. 
In Corollary~\ref{cor:gaussian-permutation}, we show that component-wise recovery (i.e., up to permutations) is possible if the covariance of the diagonal Gaussian predictor is context-dependent, similar to previous results \citep{khemakhem2020variational,sorrenson2020disentanglement}.
In Theorem~\ref{thm:approximate-readout} we provide an approximate recovery result, which drops the assumption of perfect matching.
The proof is related to previous approximate recovery theory \citep{klindt2026does}, and comes with the limitation that it requires matching latent dimensions ($d_z=d_s$).
This result establishes a bound on the average prediction error, which decreases with maximization of the goal function (\Eqref{eq:infoLDM}), but increases with lower predictability of the signal.

\section{Simulation experiments}

\subsection{Numerical evaluation}
\label{sec:numerical}

To numerically verify our proof, we performed a range of control experiments.
First, we followed a common recovery experimental setup \citep[e.g.,][]{zimmermann_contrastive_2022}.
Specifically, we generated length-five sequences with a signal following
\[
s_{t+1}=\rho R s_t+\sqrt{1-\rho^2}\,\epsilon_t,
\qquad \epsilon_t\sim\mathcal N(0,I),
\]
where $\rho=0.9$ and $R$ is a fixed random orthogonal matrix. 
At every step, we independently sampled a 15-dimensional Gaussian nuisance variable and map the combined signal and nuisance through a fixed injective 3-layer \ac{MLP} with leaky ReLU activation to obtain a 100-dimensional observation. 
We varied the signal dimension $d_S$ from 5 to 20.

To recover the signal variables we trained a 5-layer \ac{MLP} mapping the input to a latent space with same dimension as the signal, $d_Z=d_S$.
Predictions were generated through an \ac{LSTM} taking all previous timesteps and a linear prediction head.
Training maximized the InfoLDM loss for Gaussian predictors (Equation~\ref{eq:gaussloss}), where the entropy was estimated through either kNN, KDE, or log-determinant of covariance (logdet) entropy estimators, as specified in \cite{mikulasch2026understanding}, or the Stein entropy gradient estimator (Appendix~\ref{app:stein}).
We note that kNN, KDE, and Stein entropy estimators are exact in the limit of infinite samples, while logdet requires Gaussian distributed latents for this, which is correct for our data.
Simulations showed good affine signal recovery for logdet and Stein estimators, while kNN and KDE entropy estimation degraded for higher dimensional latents (Table~\ref{tab:numerical}).
Affine signal recovery also remained high with overcomplete representations (Appendix, Table~\ref{tab:randstep-latent-dimension}).

We further tested robustness of the methods to (soft) violations of the assumptions of our proof which we report in the Appendix.
Specifically, we tested robustness to non-Gaussian transition noise (Table \ref{tab:generalized-gaussian}), soft violation of nuisance privacy (Assumption 1; Table \ref{tab:predictable-nuisance}), and soft violation of condition diversity (Assumption 4; Table \ref{tab:condition-diversity-diagonal}).
The methods mostly demonstrated robust against these ablations, with signals being predictable from the context (Assumption 4) proving most critical for successful recovery.

\begin{table}[ht]
    \small
    \caption{$R^2$ of the learned representations with the true factors of variation for the numerical test, averaged over 5 runs plus minus standard deviations. We highlight well learned factors ($R^2>0.9$).}
    \label{tab:numerical}
    \vspace{3mm}
    \centering
    \begin{tabular}{lcccc}
\toprule
\textbf{Entropy estimator} & $d_s=5$ & $d_s=10$ & $d_s=15$ & $d_s=20$ \\
\midrule
KDE & $0.83 \pm 0.31$ & $0.64 \pm 0.03$ & $0.52 \pm 0.05$ & $0.44 \pm 0.03$ \\
kNN & $\mathbf{0.99 \pm 0.00}$ & $\mathbf{0.92 \pm 0.05}$ & $0.83 \pm 0.05$ & $0.67 \pm 0.04$ \\
logdet & $\mathbf{1.00 \pm 0.00}$ & $\mathbf{0.98 \pm 0.04}$ & $\mathbf{1.00 \pm 0.00}$ & $\mathbf{0.98 \pm 0.02}$ \\
Stein & $\mathbf{1.00 \pm 0.00}$ & $\mathbf{0.99 \pm 0.00}$ & $\mathbf{0.99 \pm 0.00}$ & $\mathbf{0.99 \pm 0.00}$ \\
\bottomrule
\end{tabular}

\end{table}

\subsection{Identifying stochastic causal relations between images}
\label{sec:c3d}

As second example we tested signal recovery in high-dimensional images with strong, structured signal stochasticity, and structured, but independent nuisance.
For this purpose we constructed image pairs in which a designated subset of generative factors was causally related. 
We based this on the Causal3DIdent dataset \citep{kugelgen_self-supervised_2022}, which provides 11 image factors (Figure~\ref{fig:c3d-pairs}).
After mapping these factors to standard-normal coordinates, we sampled
$s_2 = \rho s_1 + \sqrt{1-\rho^2}\,\epsilon$
with
$\epsilon\sim\mathcal N(0,I)$ and
$\rho=0.9$. Since Causal3DIdent provides a finite set of rendered images, we selected the second image uniformly from the $k=16$ nearest neighbors of $s_2$. All remaining factors, including object class, are thus random and determined by the selected image. We considered two settings: an \emph{object} signal comprising object position, rotation, and hue, and an \emph{environment} signal comprising spotlight position, spotlight hue, and background hue (Figure~\ref{fig:c3d-pairs}).
\enlargethispage{2mm}

\begin{table}[h]
    \caption{$R^2$ of the learned representations with the true factors of variation for the stochastic Causal3DIdent dataset, averaged over 5 runs. Standard deviations are not displayed here but are mostly small ($<0.02$) except for unidentified signal / identified nuisance variables, indicating intermittent optimization failures. We highlight well learned factors ($R^2>0.9$) in bold and signal factors with grey background.}
    \vspace{3mm}
    \label{tab:c3d}

    \tiny
    \centering
\begin{tabular}{llcccccccccc}
\toprule
Scenario & Entr. est. & $p_{\text{obj-}x}$ & $p_{\text{obj-}y}$ & $p_{\text{obj-}z}$ & $\theta_{\text{obj-}\alpha}$ & $\theta_{\text{obj-}\beta}$ & $\theta_{\text{obj-}\gamma}$ & $h_\text{obj}$ & $p_\text{spot}$ & $h_\text{spot}$ & $h_\text{back}$ \\
\midrule
Env. signal & KDE & $0.22$ & $0.38$ & $0.33$ & $0.03$ & $0.07$ & $0.10$ & $0.42$ & \cellcolor{gray!15} $\mathbf{0.97}$ & \cellcolor{gray!15} $\mathbf{0.95}$ & \cellcolor{gray!15} $\mathbf{0.96}$ \\
Env. signal & kNN & $0.08$ & $0.44$ & $0.39$ & $0.08$ & $0.03$ & $0.04$ & $0.07$ & \cellcolor{gray!15} $\mathbf{0.97}$ & \cellcolor{gray!15} $\mathbf{0.95}$ & \cellcolor{gray!15} $\mathbf{0.96}$ \\
Env. signal & logdet & $-0.02$ & $-0.01$ & $0.00$ & $-0.01$ & $-0.00$ & $-0.01$ & $-0.01$ & \cellcolor{gray!15} $\mathbf{0.99}$ & \cellcolor{gray!15} $0.57$ & \cellcolor{gray!15} $\mathbf{0.99}$ \\
Env. signal & Stein & $0.04$ & $0.64$ & $0.38$ & $0.07$ & $0.03$ & $0.05$ & $0.43$ & \cellcolor{gray!15} $\mathbf{0.97}$ & \cellcolor{gray!15} $\mathbf{0.94}$ & \cellcolor{gray!15} $\mathbf{0.97}$ \\
\midrule
Object signal & KDE & \cellcolor{gray!15} $\mathbf{0.98}$ & \cellcolor{gray!15} $\mathbf{0.97}$ & \cellcolor{gray!15} $\mathbf{0.97}$ & \cellcolor{gray!15} $\mathbf{0.93}$ & \cellcolor{gray!15} $\mathbf{0.96}$ & \cellcolor{gray!15} $0.83$ & \cellcolor{gray!15} $\mathbf{0.91}$ & $-0.01$ & $0.03$ & $0.24$ \\
Object signal & kNN & \cellcolor{gray!15} $\mathbf{0.97}$ & \cellcolor{gray!15} $\mathbf{0.96}$ & \cellcolor{gray!15} $\mathbf{0.96}$ & \cellcolor{gray!15} $0.86$ & \cellcolor{gray!15} $\mathbf{0.92}$ & \cellcolor{gray!15} $0.24$ & \cellcolor{gray!15} $\mathbf{0.91}$ & $-0.01$ & $-0.01$ & $0.54$ \\
Object signal & logdet & \cellcolor{gray!15} $\mathbf{0.98}$ & \cellcolor{gray!15} $\mathbf{0.97}$ & \cellcolor{gray!15} $\mathbf{0.97}$ & \cellcolor{gray!15} $\mathbf{0.93}$ & \cellcolor{gray!15} $\mathbf{0.94}$ & \cellcolor{gray!15} $\mathbf{0.91}$ & \cellcolor{gray!15} $0.86$ & $0.15$ & $-0.01$ & $-0.01$ \\
Object signal & Stein & \cellcolor{gray!15} $\mathbf{0.98}$ & \cellcolor{gray!15} $\mathbf{0.98}$ & \cellcolor{gray!15} $\mathbf{0.97}$ & \cellcolor{gray!15} $\mathbf{0.95}$ & \cellcolor{gray!15} $\mathbf{0.96}$ & \cellcolor{gray!15} $\mathbf{0.95}$ & \cellcolor{gray!15} $\mathbf{0.92}$ & $-0.01$ & $-0.01$ & $0.37$ \\
\bottomrule
\end{tabular}
\end{table}

For signal recovery we used the same architecture and losses as in the numerical examples (Section~\ref{sec:numerical}), except for switching to a ResNet-18 as the image encoder.
To give models the option to encode also nuisance we chose the latent dimension to be 10, similar to the total number of signal and nuisance factors.
The simulations showed that learning consistently resulted in affine recovery of most of the causal factors, while performance depended on the entropy estimator, with Stein achieving most consistent recovery in this scenario (Table~\ref{tab:c3d}).
We further found that the entropy estimator also influenced whether or not nuisance variables were retained in the latent representation.
Especially object position tended to be partly retained, which is sensible given the convolutional encoder; but also other variables were partly retained, an effect that became even more pronounced when computing nonlinear $R^2$ (Appendix, Table~\ref{tab:c3dnonlinear}).

\begin{figure}[tb]
    \centering
    \scriptsize
    \begin{minipage}[c]{0.42\linewidth}
        \centering
    \textbf{Environment signal}\par\vspace{1.5mm}
    \end{minipage}
    \begin{minipage}[c]{0.56\linewidth}
        \centering
    \textbf{Examples}\par\vspace{1.5mm}
    \end{minipage}
    
    \begin{minipage}[c]{0.42\linewidth}
        \centering
        % \small Latent factors\par\vspace{1mm}
        \resizebox{0.96\linewidth}{!}{\cslthreefactorpanel{ps,hs,hb}}
    \end{minipage}\hfill
    \begin{minipage}[c]{0.56\linewidth}
        \centering
        % \small Rendered images\par\vspace{1mm}
        \includegraphics[width=0.96\linewidth,trim=0 0 0 0,clip]{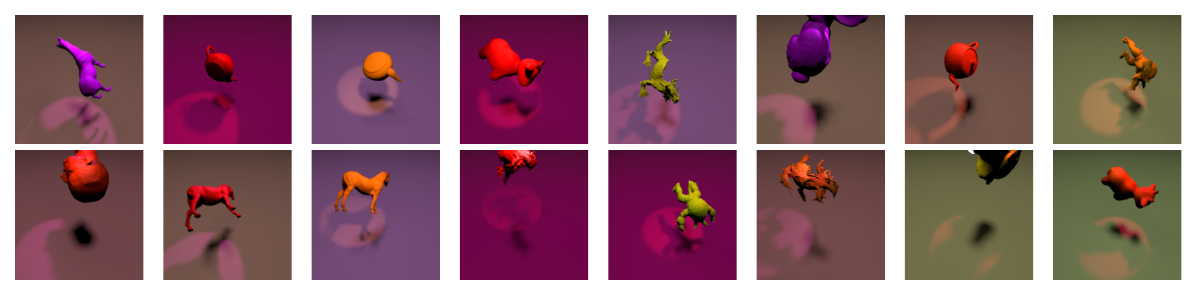}
    \end{minipage}

    \vspace{3mm}
    \begin{minipage}[c]{0.42\linewidth}
        \centering
    \textbf{Object signal}\par\vspace{1.5mm}
    \end{minipage}
    \begin{minipage}[c]{0.56\linewidth}
        \centering
    \textbf{Examples}\par\vspace{1.5mm}
    \end{minipage}

    \begin{minipage}[c]{0.42\linewidth}
        \centering
        % \small Latent factors\par\vspace{1mm}
        \resizebox{0.96\linewidth}{!}{\cslthreefactorpanel{p,theta,ho}}
    \end{minipage}\hfill
    \begin{minipage}[c]{0.56\linewidth}
        \centering
        % \small Rendered images\par\vspace{1mm}
        \includegraphics[width=0.96\linewidth,trim=0 0 0 0,clip]{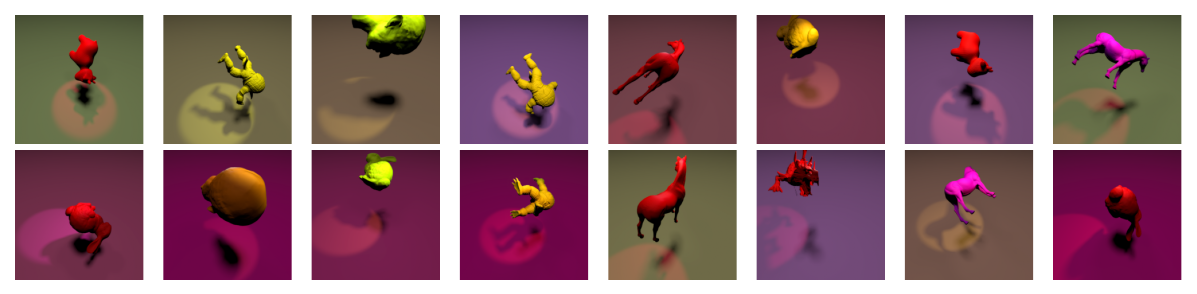}
    \end{minipage}
    \caption{Causal3DIdent image-pairs.
    The left diagrams show the true latent factors which are: object identity ($i_\mathrm{obj}$), position ($p_{\text{obj-}x}$, etc., here displayed grouped), rotation ($\theta_{\text{obj-}\alpha}$, etc.) and hue ($h_\text{obj}$), spotlight position ($p_\text{spot}$), hue ($h_\text{spot}$), and background hue ($h_\text{back}$).
    % Object identity is randomly sampled for every image.
    Arrows mark the factors causally coupled across image pairs.
    Right panels show example rendered image pairs.
    Even though factors between images are causally connected, there still exists significant variation due to noisy coupling.}
    \label{fig:c3d-pairs}
\end{figure}

\subsection{Latent world model with actions and nuisance}

In a last experiment we tested recovery and nuisance removal in a physical system with controlled actions.
To that end we extended the MuJoCo suite \citep{todorov2012mujoco}.
We chose the hopper environment, where actions were defined by a pretrained controller \citep{nikulin2025latent}, and for simulation diversity we added small Gaussian noise on both environment dynamics and controller actions.
To introduce nuisance variables we sampled 12-dimensional independent Gaussian nuisance vectors per timestep that were then translated into nuisances such as color changes of the physical body and background contrast, as well as additional high-dimensional background colored noise (Figure~\ref{fig:hopper}A,B).
We expected that latent representations would capture hopper state variables that are causally related through the simulation.

We used the same setup for learning as for the image dataset (Section~\ref{sec:c3d}) with a 16-dimensional latent state, while replacing the prediction head with an \ac{MLP} that received the \ac{LSTM} output and the encoded action.
All models learned representations that allowed to linearly decode the hopper state (Figure~\ref{fig:hopper}D), while nonlinear decoding further improved $R^2$, depending on the entropy estimator (Appendix, Figure~\ref{fig:hopper_R2_appendix}). The nonlinear recovery probe is reasonable given that the environment does not explicitly follow the Gaussian dynamics of Corollary~\ref{cor:gaussian-readout} (see Appendix~\ref{app:sim}). 
Additional analysis also showed that increasing the latent dimensionality further improved linear readouts, corroborating a slightly nonlinear encoding (Appendix, Table~\ref{tab:hopper-latent-dimension}).
Together with the numerical results in Table~\ref{tab:randstep-latent-dimension}, this suggests choosing overcomplete latent representations to reduce the risk of discarding signal dimensions.

We also trained models with an image decoder with Gaussian noise model (i.e, L2 loss) and attached gradient.
This setup is similar in spirit to generative identifiable models such as iVAE \citep{khemakhem2020variational}, which have an explicit decoder but also a latent loss that allows to identify underlying variables---purely reconstructive models without latent predictive loss are not expected to learn useful latent variables \citep{hyvarinen2019nonlinear}.
The additional decoder loss resulted in representations that did not allow to decode hopper state, but consistently encoded parts of the nuisance variables (Appendix, Figure~\ref{fig:hopper_R2_appendix}).
As additional control we trained a model with the same predictor setup and an image decoder to predict the next observation, which did not result in better performance than direct decoding (Figure~\ref{fig:hopper}D).
This discrepancy can also be seen in the different topological structure of the latent representations (Figure~\ref{fig:hopper}C).
Models with image reconstruction/prediction fail in our simulations because they have to form an explicit model of observation nuisance.
For the complex, multimodal nuisance in the data this is impossible with the simple Gaussian decoder we used. 
While more flexible decoder models might fare better, they would have to learn the nuisance model from data explicitly, which we show is not required to model the physical system.
% Non-generative models with detached decoder gradient did not consistently show this behavior, although we found that training with the KDE estimator lead to higher nuisance encoding at the cost of lower signal decoding performance.

\begin{figure}[tb]
    \centering
    \includegraphics[width=1\textwidth]{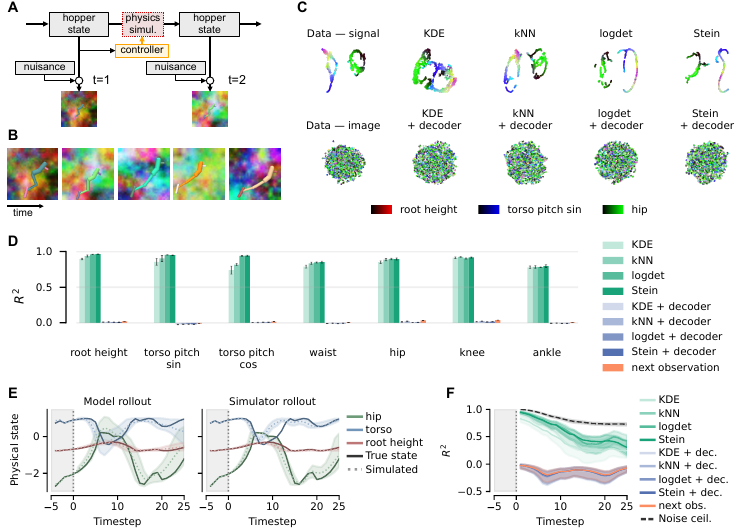}
    \caption{Learning a world model for the controlled MuJoCo hopper.  
    \textbf{(A)}~Dataset generation.
    \textbf{(B)}~Example timeseries.
    \textbf{(C)}~UMAP of true signal (hopper state) and learned representations. Colors denote variable values.
    \textbf{(D)}~$R^2$ of latent representations with hopper variables from linear readout.
    \textbf{(E)}~Example model rollout of physical states based on the logdet model with nonlinear decoder given 5 conditioning frames. We compare to true simulator rollouts for the same trajectory. Shaded areas denote 65\% sampling based intervals.
    \textbf{(F)}~Rollout prediction $R^2$ for different methods with nonlinear decoder. Noise ceiling denotes $R^2$ of true simulator rollout mean. Shaded areas denote 95\% bootstrapping intervals.}
    \label{fig:hopper}
\end{figure}

Finally we demonstrated that the learned latent model can simulate the physical hopper system.
We rolled out trajectories by sampling next state predictions as $z_t\sim p_\theta(\cdot|z_c,a)$ ($N=1000$), where $z_c$ is based on the past embeddings and we use the true action $a$, feed them back into the model as the next $z_c$, and then decode the physical state (Figure~\ref{fig:hopper}E).
While the model was trained on sequences of length 8, it generalizes to longer predictions.
Sampling also allows to directly read out rollout uncertainty as sampling intervals of trajectories, which shows that uncertainty increases over time.
Compared to using the true simulator for sampling rollouts we find that the model uncertainty is larger, presumably because it includes effects of observation nuisance and imperfect model learning, additionally to the stochasticity in the signal.
When measuring the correlation of the rollout with future timesteps the model $R^2$ decreases slightly faster than the true system would allow for, but nevertheless drastically outperforms generative model variants (Figure~\ref{fig:hopper}F).

\section{Conclusion and Limitations}

We proved that common \ac{SSL} learning methods recover signal variables even in the presence of high-dimensional nuisance and stochastic latent dynamics, closing the gap between prior identifiability results, which considered either but not both. 
This answers the question posed at the outset: representations can formally distinguish unpredictable nuisance from mere stochastic variation in the predictable signal, because 
predictive \ac{MI} maximization and \ac{DM} jointly constrain what is retained and how.
Our result requires only few assumptions, the strongest of which are exponential family conditional distributions in the true latent variables, and an algorithm that successfully maximizes the InfoLDM loss;
otherwise, the latent dynamics can be almost arbitrarily complicated, and nuisance can affect observations nontrivially.
Future work should explore the viability and real-world relevance of  predictive exponential family distributions beyond the Gaussian scenario tested here.

Our theory does not preclude that nuisance is encoded in the learned representation (orthogonally to the signal) if there is capacity to accommodate it, and entropy maximization in principle should encourage this \citep{kugelgen_self-supervised_2022,mikulasch2026understanding}.
Empirically we found that nuisance variables are often not decodable even if the latent dimensionality would permit it, either because they collapse, or because they are represented in a highly nonlinear manner that is not easily inverted.
We also found that both signal recovery performance and nuisance leakage depended on the choice of entropy estimator, and the data it was applied to.

Our results highlight that improved entropy estimation is one of the most pressing open problems towards robust and generally applicable \ac{SSL} methods.
Throughout our simulations, performance was critically dependent on the estimators used, with logdet and Stein estimators consistently outperforming the other methods.
However, the logdet estimator assumes Gaussian distributed latent variables, which was approximately true for our data, but does not have to hold in general.
This suggests that Stein entropy gradient estimation, or potentially other methods that do not have to explicitly estimate the latent marginal distribution, are preferable as general purpose anti-collapse mechanisms.

% A question we did not answer is whether encoding nuisance variables or not is actually harmful for performance.
% Our empirical results suggest that there is a tradeoff, but it is not clear if these insights translate from the dimension-restricted setting here to typical learning setups with highly overcomplete latent spaces.
% Nevertheless, especially in the small labelled data regime, it is likely that additional latent nuisance representation produces excess generalization error.

% Seems to collide with previous arguments that through entropy f becomes invertible on the data manifold.
% However, this is not the case. 
% We here show that in the optimum limit signal variables are guaranteed to be recovered, but this makes no statement about nuisance variables.
% Indeed, previous arguments indicate that entropy would try to preserve them, although there is no guarantee that they be simple to read out.
% One could introduce additional pressure to remove nuisance, but we argue that this is generally undesirable. 
% First, many methods such as Bottleneck most likely make the nuisance unpredictable but not necessarily remove it. 
% Second, most of the time it is actually desirable to retain this information, since not all variables of interest always are signals.

\subsection*{AI use statement}

In this work, we used generative AI tools to assist in the writeup of proofs, correcting proofs, and to create or edit software code.
In particular, Theorem 1 was proposed and proved by the authors and corrected and checked with ChatGPT 5.5.
The approximate recovery result in Theorem 2 was developed with the help of ChatGPT 6, which derived the bound and conducted the proof; the authors proposed the assumptions, and proofread and simplified the proof.
We have not used generative AI tools to help develop theoretical models or conceptual frameworks, propose or refine hypotheses, design or provide feedback on research methodology or experiments, implement methods and generate synthetic data sets, or to interpret results.
We have reviewed all AI-assisted work. LLM-generated code was verified and tested for correctness, the main Theorem \ref{thm:exponential-readout} was checked and formalized in Lean and the Lean statement critically reviewed. We take responsibility for the final content of this work,
including text, claims or artifacts produced with the aid of generative AI.

\subsection*{Ethics statement}

This paper presents work whose goal is to advance the field
of machine learning. There are many potential societal
consequences of our work, none of which we feel must be
specifically highlighted here.

\subsection*{Reproducibility statement}

For reproducibility we provide all simulation and evaluation code, as well as Lean proof formalizations at \href{https://github.com/fmi-basel/identifiable-stochastic-nuisance}{https://github.com/fmi-basel/identifiable-stochastic-nuisance}.


\subsubsection*{Acknowledgments}
We thank Michael Hauri and all Zenke
Lab members for their input and discussions. This project
was supported by the Swiss National Science Foundation
(Grant Number PCEFP3\_202981) and the Novartis Research
Foundation.

\bibliography{iclr2027_conference}
\bibliographystyle{preprintbiblio}

\newpage

\appendix
\section{Appendix}

\subsection{Mathematical appendix}

\subsubsection{Detailed assumptions}
\label{app:assumptions}

% For completeness we repeat the generative model structure here.

% \paragraph{Generative model}
% Let $s\in\mathcal S\subseteq\R^{d_s}$ be the signal variable and let $n\in\R^{d_n}$ be a nuisance variable.
% Let $c$ denote the underlying true condition, $x_c$ the observed condition, and $z_c=f_c(x_c)\in\mathcal Z\subseteq\R^{d_{z_c}}$ the learned encoded condition.
% In the proof, to simplify notation, we will combine the condition generator and encoder into $p(z_c\mid c)$.
% Let $x=g(s,n)$ be the observation and $z=f(x)\in\mathcal Z\subseteq\R^{d_z}$ be the learned code (Figure~\ref{fig:private-nuisance-dag}). 

\begin{assumption}[Private nuisance]
\label{ass:markov}
The stochastic part of the model factorizes as
\begin{equation}
    p(n,s,x_c,c)
    =p(n\mid s)p(s\mid c)p(x_c\mid c)p(c),
    \label{eq:markov-factorization}
\end{equation}
with $x=g(s,n)$, $z=f(x)$, and $z_c=f_c(x_c)$.
More specifically, the nuisance is private to the current observation in the sense that
\begin{equation}
    n\indep c\mid s.
    \label{eq:private-nuisance}
\end{equation}
\end{assumption}

The privacy condition has the equivalent information-theoretic form
\begin{equation}
    I[(s,n);c]
    =I[s;c]+I[n;c\mid s]
    =I[s;c].
    \label{eq:private-nuisance-information}
\end{equation}
Thus the nuisance carries no information about the condition beyond the signal.

\begin{assumption}[Exponential-family prediction]
\label{ass:exponential-families}
Let $\tau_\star:\mathcal S\to\R^{r_s}$ be a sufficient statistic, and let $h_s(s)$ be a fixed carrier function on $\mathcal S$.
For each condition $c$, let $\alpha(c)\in\R^{r_s}$ be its natural parameter, and let $\psi$ be the corresponding log-partition function.
The true predictive family is
\begin{equation}
    p(s\mid c)
    =h_s(s)
    \exp\!\left(
        \alpha(c)^\top\tau_\star(s)-\psi(\alpha(c))
    \right)
    \label{eq:source-family}
\end{equation}
for almost every condition $c$.

Similarly, let $\tau_z:\mathcal Z\to\R^{r_z}$ be a sufficient statistic, and let $h_z(z)$ be a fixed carrier function on $\mathcal Z$.
For each encoded condition $z_c$, let $\beta(z_c)\in\R^{r_z}$ be its learned natural parameter, and let $\phi$ be the corresponding log-partition function.
The learned predictive model is
\begin{equation}
    p_\theta(z\mid z_c)
    =h_z(z)
    \exp\!\left(
        \beta(z_c)^\top\tau_z(z)-\phi(\beta(z_c))
    \right)
    \label{eq:learned-model-family}
\end{equation}
for almost every encoded condition $z_c$.
\end{assumption}
As mentioned before, it is not necessary to assume any particular marginal distribution of the encoded conditioning variables $z_c$.
In the joint formulation, one may set $p_\theta(z_c)=q_f(z_c)$, so the KL term compares only the predictive conditionals (Equation~\ref{eq:infoLDMsimple}).

\begin{assumption}[Distribution matching and information saturation]
\label{ass:ideal-optimum}
Let $q_f(z,z_c)$ be the encoded joint distribution and let $p_\theta(z,z_c)$ be a learned predictive model.
Information-augmented predictive LDM uses the population objective
\begin{equation}
    \mathcal F(f,f_c,\theta)
    =
    -\KL[q_f(z,z_c)\|p_\theta(z,z_c)]+I_{q_f}[z;z_c].
    \label{eq:mi-objective}
\end{equation}
We assume that an ideal population optimum is realizable through optimizing the model $p_\theta$, encoder $f$, and the condition encoder $f_c$, at which
\begin{equation}
    \KL[q_f(z,z_c)\|p_\theta(z,z_c)]=0,
    \qquad
    I_{q_f}[z;z_c]=I[(s,n);c]=I[s;c]<\infty.
    \label{eq:ideal-optimum}
\end{equation}
\end{assumption}

As discussed in the main text, this assumption implicitly replaces common invertibility assumptions.
It also implies that latent distributions match, and for the proof we define
\begin{equation}
    p(z,z_c):=q_f(z,z_c)=p_\theta(z,z_c).
\end{equation}

% This assumption has two central consequences.
% First, minimizing the KL term leads to distribution matching of the encoded distribution $q_f$ to the predictive model $p_\theta$, i.e., $q_f(z,z_c)=p_\theta(z,z_c)$.
% Together with Assumption~\ref{ass:exponential-families} this gives
% \begin{equation}
%     \begin{aligned}
%     p(z\mid z_c)
%     &:=q_f(z\mid z_c)
%       =p_\theta(z\mid z_c)\\
%     &=h_z(z)\exp\!\left(
%         \beta(z_c)^\top\tau_z(z)-\phi(\beta(z_c))
%     \right).
%     \end{aligned}
%     \label{eq:learned-family}
% \end{equation}
% Thus, if the model $p_\theta$ is exponential family, so is $q_f$.
% Second, data processing and Assumption~\ref{ass:markov} give
% \begin{equation}
%     I[z;z_c]
%     \le I[x;x_c]
%     \le I[s;c]
%     =I[(s,n);c],
%     \label{eq:mi-upper-bound}
% \end{equation}
% and Assumption~\ref{ass:ideal-optimum} makes every inequality an equality.
% Thus, the observed condition $x_c$ must contain all information relevant to predicting $s$, and similarly $x$ must preserve all this predictable information.
% This implicitly replaces common invertibility assumptions on the generative model.
% It further removes the need for having direct access to conditioning variables \citep[e.g.,][]{khemakhem2020ice}, they can be accessed through a learned encoding with nuisance as well.
% Last, it implicitly requires that the latent representation is sufficiently high-dimensional to capture the required information, as will be demonstrated by Theorem~\ref{thm:exponential-readout}.

\begin{assumption}[Condition diversity]
\label{ass:condition-diversity}
Even after excluding an arbitrary zero-probability subset of conditions, one can choose $c_0,c_1,\ldots,c_{r_s}$ such that
\begin{equation}
    D
    :=
    \begin{bmatrix}
        (\alpha(c_1)-\alpha(c_0))^\top\\
        \vdots\\
        (\alpha(c_{r_s})-\alpha(c_0))^\top
    \end{bmatrix}
    \in\R^{r_s\times r_s}
    \label{eq:parameter-difference-matrix}
\end{equation}
is invertible.
\end{assumption}
Invertibility requires the condition variation to span all $r_s$ natural-parameter directions.
In the proof, these directions supply $r_s$ independent likelihood-ratio equations that recover every coordinate of the source sufficient statistic.

\begin{assumption}[No affine redundancy in the source statistic]
\label{ass:source-nondegenerate}
For every condition $c$ for which \Eqref{eq:source-family} holds, the random vector $\tau_\star(s)$ with $s\sim p(s\mid c)$ is not almost surely contained in a proper affine hyperplane of $\R^{r_s}$.
\end{assumption}
This assumption complements Assumption~\ref{ass:condition-diversity} and says that the sufficient statistic also genuinely uses all $r_s$ coordinates.

\subsubsection{Full Proof of Theorem~\ref{thm:exponential-readout}}
\label{app:proof}

\begingroup
\def\thetheorem{\ref{thm:exponential-readout} (Restated)}

\begin{theorem}[Exponential-family readout in InfoLDM]
Suppose Assumptions~\ref{ass:markov},~\ref{ass:exponential-families},~\ref{ass:ideal-optimum},~\ref{ass:condition-diversity}, and~\ref{ass:source-nondegenerate} hold, and the InfoLDM goal function (\Eqref{eq:infoLDM}) is fully maximized.

Then the dimension of the model sufficient statistic $r_z\ge r_s$, and there is a full-row-rank matrix $A\in\R^{r_s\times r_z}$ and a vector $b\in\R^{r_s}$ such that
\begin{equation}
    \boxed{\tau_\star(s)=A\tau_z(z)+b}
    \qquad\text{almost surely}.
    \label{eq:main-readout2}
\end{equation}
If $\tau_\star$ is injective, then $s$ is almost surely a measurable function of $z$.
If $r_z=r_s$, then $A$ is invertible and the two sufficient statistics are related by an invertible affine transformation.
\end{theorem}
\addtocounter{theorem}{-1}
\endgroup

\begin{proof}

The proof proceeds in four steps. 
In  Step~1, we show that \ac{MI} maximization implies that the learned context $z_c$ and true context $c$ are interchangeable for latent prediction.
In Step~2, we show that it also implies that it is possible to equate the likelihood ratios of the learned representation $z$ and the of true signal $s$ with respect to some reference condition $c_0$.
This is the key identity that explains why representations have to recover the signal.

The following two steps are closely related to previous exponential family identifiability proofs \citep{khemakhem2020variational}.
In Step~3 we use the above likelihood ratio, the exponential family assumption, and the condition diversity assumption to establish that there is an affine relation between the statistics $\tau_z(z)$ and $\tau_\star(s)$ when $z$ and $s$ are sampled under $c_0$.
Finally, we extend this result in Step~4 to the entire $p(c)$, and derive the full-rank and invertibility relating $z$ and $s$.

\noindent\emph{Step 1: Information saturation and transfer.}
The chain rule of mutual information lets us compute
\begin{align}
\begin{split}
    \label{eq:MIequality}
    I[s;c]-I[z;z_c]
    &=\bigl(I[s;c]-I[z;c]\bigr)
      +\bigl(I[z;c]-I[z;z_c]\bigr)\\
    &=\bigl(I[s,z;c]-I[z;c]\bigr)
      +\bigl(I[z;c,z_c]-I[z;z_c]\bigr)\\
    &=I[s;c\mid z]+I[z;c\mid z_c].
\end{split}
\end{align}
The second step follows, since from the generative model (Figure~\ref{fig:private-nuisance-dag}) and Assumption~\ref{ass:markov} we know that $z\indep c\mid s$ and $z\indep z_c\mid c$.
\Eqref{eq:MIequality}, together with Assumption~\ref{ass:ideal-optimum} implies $s\indep c\mid z$ and $z\indep c\mid z_c$.
Consequently, for almost every jointly occurring pair $(c,z_c)$,
\begin{equation}
    p(z\mid c)
    =p(z\mid c,z_c)
    =p(z\mid z_c).
    \label{eq:conditional-transfer}
\end{equation}
Thus $p(z\mid c)$ is a member of the learned exponential family for almost every condition $c$.

\medskip
\noindent\emph{Step 2: Likelihood-ratio equality.}
The two conditional independences involving $s$ tell us that the conditional distributions relating $s$ and $z$ are independent of $c$, i.e., $p(z\mid s,c)=p(z\mid s)$ and $p(s\mid z,c)=p(s\mid z)$.
Therefore, the following two factorizations hold for almost every condition $c$:
\begin{equation}
    p(s,z\mid c)
    =p(z\mid s)p(s\mid c)
    =p(s\mid z)p(z\mid c).
    \label{eq:common-factorizations}
\end{equation}
Because the conditional distributions between $s$ and $z$ in \Eqref{eq:common-factorizations} do not depend on $c$, they cancel when the distributions at $c$ and some reference $c_0$ are compared
\begin{equation}
    \frac{p(s\mid c)}{p(s\mid c_0)}
    =\frac{p(z\mid c)}{p(z\mid c_0)}
    \qquad\text{for almost every $(s,z)$ generated under $c_0$}.
    \label{eq:rn-cancellation-any}
\end{equation}
The ratios are well defined because all selected family members are strictly positive wherever their common carrier functions are positive.

\medskip
\noindent\emph{Step 3: Locally affine relation.}
We now want to select a number of $r_s$ different contexts $c_i$ that constrain the relation between $s$ and $z$.
By using \Eqref{eq:rn-cancellation-any} for any $c_i$ we get
\begin{equation}
    \frac{p(s\mid c_i)}{p(s\mid c_0)}
    =\frac{p(z\mid c_i)}{p(z\mid c_0)}
    \qquad\text{for almost every $(s,z)$ generated under $c_0$}.
    \label{eq:rn-cancellation}
\end{equation}
Taking logarithms, and using the two exponential-family forms (Assumption~\ref{ass:exponential-families}) we get
\begin{equation}
        (\alpha(c_i)^\top - \alpha(c_0)^\top)\tau_\star(s)-[\psi(\alpha(c_i))-\psi(\alpha(c_0))]
    =
        (\beta_i - \beta_{0})^\top\tau_z(z)-[\phi(\beta_i) - \phi(\beta_{0})].
    \label{eq:log-ratios}
\end{equation}
Here we chose parameters $\beta_i$ for which (by the result of Step 1)
\begin{equation}
    p(z\mid c_i)
    =h_z(z)\exp\!\left(\beta_i^\top\tau_z(z)-\phi(\beta_i)\right).
\end{equation}
We now define $d_i:=\alpha(c_i)-\alpha(c_0)$, and $e_i:=\beta_i-\beta_0$.
With this notation \Eqref{eq:log-ratios} becomes
\begin{equation}
    d_i^\top\tau_\star(s)
    =e_i^\top\tau_z(z)+h_i,
    \qquad
    h_i:=\psi(\alpha(c_i))-\psi(\alpha(c_0))
          -\phi(\beta_i)+\phi(\beta_0).
    \label{eq:selected-log-ratios}
\end{equation}
Let $D$ have rows $d_i^\top$, $E$ have rows $e_i^\top$, and let $h=(h_1,\ldots,h_{r_s})^\top$.
Thus, stacking \Eqref{eq:selected-log-ratios} gives
\begin{equation}
    D\tau_\star(s)=E\tau_z(z)+h
    \qquad\text{for almost every $(s,z)$ generated under $c_0$}.
    \label{eq:stacked-ratios}
\end{equation}
Finally, Assumption~\ref{ass:condition-diversity} lets us choose conditions $c_0,c_1,\ldots,c_{r_s}$ for which $D$ is invertible.
This proves the readout for samples generated under $c_0$, with
\begin{equation}
    A:=D^{-1}E,
    \qquad
    b:=D^{-1}h.
    \label{eq:readout-coefficients}
\end{equation}

\medskip
\noindent\emph{Step 4: Globalization, rank, and reconstruction.}
What remains is to extend this reference-condition identity to the full joint distribution, prove that $A$ has full row rank, and derive the reconstruction conclusions.

For globalization, we note that for almost every condition $c$, the distribution $p(s\mid c)$ has the same zero-probability sets as $p(s\mid c_0)$ because both are strictly positive on the common carrier.
Multiplying these distributions with the same $p(z\mid s)$ preserves this property for the conditional joint distributions of $p(s,z\mid c)$, i.e., any set of probability zero under $p(s,z\mid c_0)$ has probability zero under $p(s,z\mid c)$.
In particular, the set $\{(s,z):\tau_\star(s)\neq A\tau_z(z)+b\}$ has probability zero under $p(s,z\mid c_0)$ and hence under $p(s,z\mid c)$.
The readout therefore holds conditionally for almost every condition $c$.
Averaging over $p(c)$ gives \Eqref{eq:main-readout} under the original joint distribution.

To show that $A$ is full rank, assume it were not, i.e., $\rank(A)<r_s$. In this case the reference-condition identity (\Eqref{eq:stacked-ratios}) places $\tau_\star(s)$ in the proper affine subspace $b+\im(A)$ under $p(s\mid c_0)$.
Because $p(s\mid c_0)$ is strictly positive on the common carrier, this contradicts Assumption~\ref{ass:source-nondegenerate}.
Thus $\rank(A)=r_s$ and $r_z\ge r_s$.
When $r_z=r_s$, the full-row-rank matrix $A$ is square and hence invertible.

To reconstruct $s$ from $z$ we have to invert $\tau_\star$. If $\tau_\star$ is injective, the Lusin--Souslin theorem supplies a measurable inverse on its image, which may be extended measurably by a fixed value outside that image, so
\begin{equation}
    s=\tau_\star^{-1}\!\left(A\tau_z(z)+b\right)
    \qquad\text{almost surely}.
    \label{eq:state-reconstruction}
\end{equation}
\end{proof}

\subsubsection{Gaussian specialization}

In the main text we prove affine recovery of the source signal $s$ under Gaussian conditional distributions in Corollary \ref{cor:gaussian-readout}.
One question one might ask is whether the remaining dimensions that are not contributing to the affine readout might still contain signal information.
While we can't rule this out in the general case, we can rule it out in the fixed and isotropic covariance Gaussian case.

For this theorem in particular, but also more generally, we point out that the perfect matching Assumption \ref{ass:ideal-optimum} requires $d_z\leq d_s+d_n$ to be realizable. 
Otherwise, a deterministic Lipschitz encoder cannot produce a full-dimensional representation distribution, making exact matching to the Gaussian model impossible.
Experimental results with overcomplete representations that do not follow this requirement $(d_z>d_s+d_n)$ we therefore consider empirical tests beyond the guarantees of our theory, which mainly rely on well-regularized entropy estimation.

\begin{corollary}[Gaussian signal--nuisance separation]
\label{cor:gaussian-nuisance}
Under the assumptions of Corollary~\ref{cor:gaussian-readout}
and assuming $d_z>d_s$, there exist an invertible matrix
$T\in\R^{d_z\times d_z}$ and a vector $a\in\R^{d_z}$ such that
\begin{equation}
    Tz+a=
    \begin{pmatrix}
        s\\ v
    \end{pmatrix},
    \qquad
    v\in\R^{d_z-d_s},
    \qquad
    v\indep(s,c),
    \label{eq:gaussian-nuisance-separation}
\end{equation}
where $v$ is Gaussian distributed. 

If additionally the covariance is isotropic $\Sigma_\theta=\sigma^2\identity_{d_z}$,
then the Euclidean projection onto the nullspace of the
signal readout, which does not contribute to signal decoding, is independent of the signal and context:
\begin{equation}
    (\identity_{d_z}-A^+A)z\indep(s,c),
\end{equation}
where $s=Az+b$ and
$A^+=A^\top(AA^\top)^{-1}$.
\end{corollary}

\begin{proof}
\medskip\noindent\emph{Step 1: Constructing signal--nuisance separation.}
Our goal is to construct a $T$ and $v$ such that \Eqref{eq:gaussian-nuisance-separation} holds.
Corollary~\ref{cor:gaussian-readout} gives $s=Az+b$, where
$A$ has full row rank.
Complete the rows of $A$ to an invertible matrix
$\begin{pmatrix}A\\C\end{pmatrix}$, and define $u=Cz$.
Thus $(s,u)$ is an invertible affine transformation of $z$.

Information saturation and exact distribution matching imply that
$p(z\mid c)$ is Gaussian with fixed covariance $\Sigma_\theta$,
as shown in the proof of Theorem~\ref{thm:exponential-readout}.
Consequently, $(s,u)\mid c$ is jointly Gaussian, with
\[
    \operatorname{Cov}[u,s\mid c]
    =C\Sigma_\theta A^\top,
    \qquad
    \operatorname{Cov}[s\mid c]
    =A\Sigma_\theta A^\top.
\]
Both covariances are independent of $c$.
Moreover, $A\Sigma_\theta A^\top$ is positive definite
because $A$ has full row rank and $\Sigma_\theta\succ0$.

We remove the conditional correlation between $u$ and $s$
by defining
\[
    K=C\Sigma_\theta A^\top
      (A\Sigma_\theta A^\top)^{-1},
    \qquad
    v=u-K(s-b)=(C-KA)z.
\]
This choice cancels the conditional covariance exactly,
\[
\begin{aligned}
    \operatorname{Cov}[v,s\mid c]
    &=\operatorname{Cov}[u,s\mid c]
      -K\operatorname{Cov}[s\mid c]\\
    &=C\Sigma_\theta A^\top
      -KA\Sigma_\theta A^\top
      =0.
\end{aligned}
\]
Since $(s,v)\mid c$ is jointly Gaussian, this implies
$v\indep s\mid c$.
The transformation preserves all information in $(s,u)$,
because $u=v+K(s-b)$. Hence
\[
    T=\begin{pmatrix}A\\C-KA\end{pmatrix},
    \qquad
    a=\begin{pmatrix}b\\0\end{pmatrix}
\]
satisfy $Tz+a=(s^\top,v^\top)^\top$, with $T$ invertible and $v$ independent of the signal given the context.

It remains to show that $v$ is independent of both $s$ and $c$.
Private nuisance gives $z\indep c\mid s$, and therefore
$v\indep c\mid s$.
Together with $v\indep s\mid c$, this yields
\[
    p(v\mid s,c)=p(v\mid c)=p(v\mid s)
\]
almost everywhere.

To see why this forces a single distribution for $v$,
consider any two contexts $c_1,c_2$ for which these
equalities hold. Both signal conditionals have strictly
positive densities on $\R^{d_s}$, so for almost every
signal value $s$,
\[
    p(v\mid c_1)=p(v\mid s)=p(v\mid c_2).
\]
Thus $p(v\mid c)$ is the same distribution for almost
every context. Averaging over $c$ identifies this common
distribution as $p(v)$.
Since $v\indep s\mid c$, we consequently obtain
\[
    p(v\mid s,c)=p(v\mid c)=p(v),
\]
which proves $v\indep(s,c)$.

\medskip\noindent\emph{Step 2: Independence of readout nullspace under isotropic covariance.}
We finally show that residual coordinates are independent of the signal if $\Sigma_\theta=\sigma^2\identity_{d_z}$.
Choose the rows of $C$ to form an orthonormal basis of
$\ker A$. That is
\[
    CA^\top=0,
    \qquad
    CC^\top=\identity_{d_z-d_s},
\]
from which it follows that
\[
    C^\top C=\identity_{d_z}-A^+A,
\]
where $A^+$ is the Moore-Penrose inverse of $A$.
This is possible because the rows of $A$ and $C$ span complementary subspaces,
so the stacked matrix remains invertible.
For this choice, and with the isotropic covariance, the regression coefficient simplifies to
\[
    K
    =C\Sigma_\theta A^\top
      (A\Sigma_\theta A^\top)^{-1}
    =CA^\top(AA^\top)^{-1}
    =0.
\]
Thus the independent nuisance coordinates constructed
above are simply $v=Cz$.
The Euclidean residual is consequently
\[
    (\identity_{d_z}-A^+A)z=C^\top Cz=C^\top v.
\]
Since $v\indep(s,c)$, every deterministic function of $v$
is also independent of $(s,c)$, proving the claim.
\end{proof}

The affine recovery result in Corollary \ref{cor:gaussian-readout} can be strengthened to recovery up to permutation and scaling through a mechanism proposed previously \citep{khemakhem2020ice,sorrenson2020disentanglement}:
If the covariance of the Gaussian varies sufficiently over the context.
Empirical tests of this result are given in Table \ref{tab:component-recovery-diagonal}.

\begin{corollary}[Component-wise recovery for diagonal Gaussian predictors]
\label{cor:gaussian-permutation}
Under the assumptions of Theorem~\ref{thm:exponential-readout}, suppose the true and learned predictive distributions are nondegenerate diagonal Gaussians with variable covariance,
\begin{equation}
\begin{aligned}
    p(s\mid c)
    &=\Normal(s;\mu_\star(c),\Sigma_\star(c)),\\
    p_\theta(z\mid z_c)
    &=\Normal(z;\mu_\theta(z_c),\Sigma_\theta(z_c)),
\end{aligned}
\label{eq:diagonal-gaussian-conditionals}
\end{equation}
where $\Sigma_\star(c)$ and $\Sigma_\theta(z_c)$ are positive definite and diagonal. 
Assumption~\ref{ass:condition-diversity} holds for the sufficient statistics of a variable-covariance Gaussian
$\tau_\star(s)=(s,s^{\odot 2})$ and
$\tau_z(z)=(z,z^{\odot 2})$, where the square is applied
element-wise.

Then $d_z\ge d_s$, and there exist distinct indices
$j(i)\in\{1,\ldots,d_z\}$ and constants $a_i\ne 0$, $b_i$
such that
\begin{equation}
    \boxed{s_i=a_i z_{j(i)}+b_i}
    \qquad\text{almost surely}
    \label{eq:gaussian-permutation}
\end{equation}
for all $i=1,\ldots,d_s$.
In particular, if $d_z=d_s$, then
$s=P\Lambda z+b$, where $P$ is a permutation matrix
and $\Lambda$ is an invertible diagonal matrix.
\end{corollary}

\begin{proof}
As in Corollary \ref{cor:gaussian-readout}, the proof follows from the structure of the Gaussian sufficient
statistics. By Theorem~\ref{thm:exponential-readout}, there are
a matrix $A$ and a vector $b$ such that
\begin{equation}
    \begin{pmatrix}s\\s^{\odot 2}\end{pmatrix}
    =
    A\begin{pmatrix}z\\z^{\odot 2}\end{pmatrix}+b
    \qquad\text{almost surely}.
    \label{eq:gaussian-statistic-relation}
\end{equation}
Consequently, each $s_i$ is a polynomial of degree at most two
in the coordinates $z_j$, and the same holds for
$s_i^2$.

Any quadratic term in $s_i$ would produce fourth-order terms
in $s_i^2$, contradicting \Eqref{eq:gaussian-statistic-relation}.
Therefore $s_i$ must be affine in $z$,
\begin{equation}
    s_i=\sum_{j=1}^{d_z}a_{ij}z_j+b_i.
\end{equation}
Squaring this expression produces cross terms
$2a_{ij}a_{ik}z_jz_k$ for $j\ne k$. Since the right-hand side
of \Eqref{eq:gaussian-statistic-relation} contains no such
terms, we must have $a_{ij}a_{ik}=0$ whenever $j\ne k$.
Hence each $s_i$ depends on at most one coordinate $z_j$.

The positive definite covariance of $p(s\mid c)$ ensures that
each $s_i$ is nonconstant and that no two signal coordinates
can be affine functions of the same $z_j$. The corresponding
indices $j(i)$ must therefore be distinct, proving
\Eqref{eq:gaussian-permutation} and $d_z\ge d_s$.
For $d_z=d_s$, these indices form a permutation, and the
relation reduces to a permutation, component-wise scaling,
and translation.
\end{proof}

\subsubsection{Approximate recovery under Gaussian dynamics}
\label{app:approximate-identifiability}

A crucial question is whether signal recovery works under imperfect
matching and \ac{MI} maximization.
While this is challenging to prove in the same generality as our main
result (Theorem \ref{thm:exponential-readout}), it is possible for a more
restricted Gaussian setup, similar to the one in
\cite{klindt2026does}.
The approximate recovery result below is more general than the one presented in \cite{klindt2026does} in that it allows for nonlinear dynamics and predictors, and view-private nuisance.

Specifically, we retain the assumptions of Corollary \ref{cor:gaussian-readout} but quantify the predictability of the true signal, relax
Assumption~\ref{ass:ideal-optimum}, and add a dimension matching
requirement as follows.

\medskip
\noindent\textbf{Assumption $2^\ast$
(Gaussian signal marginal and fixed-covariance Gaussian prediction).}
The true signal has a standard Gaussian marginal and
is conditionally Gaussian with fixed isotropic covariance,
\begin{equation}
    s\sim\mathcal N(0,I_{d_s}),
    \qquad
    p(s\mid c)
    =\mathcal N\!\left(
        s;\mu_\star(c),(1-\rho^2)I_{d_s}
      \right),
    \qquad 0<\rho<1.
\label{eq:approx-gaussian-source}
\end{equation}
Here $\rho^2$ denotes the fraction of signal variance predictable
from $c$ in each direction.
The learned predictor is a fixed-covariance Gaussian,
\begin{equation}
    p_\theta(z\mid z_c)
    =\mathcal N(z;\mu_\theta(z_c),\Sigma_\theta),
    \qquad \Sigma_\theta\succ0.
\label{eq:approx-gaussian-predictor}
\end{equation}

\medskip
\noindent\textbf{Assumption $3^\ast$ (Approximate optimum).}
The population InfoLDM gap
\begin{equation}
\begin{split}
    \Delta
    &:=I[s;c]-\mathcal F(f,f_c,\theta)\\
    &=\mathrm{KL}[q_f(z,z_c)\|p_\theta(z,z_c)]
      +I[s;c]-I_{q_f}[z;z_c]
\end{split}
\label{eq:approx-gap}
\end{equation}
is finite.
Both terms in the second line are nonnegative by
Assumption~\ref{ass:markov} and data processing.
Under Assumption~$2^\ast$,
\[
    I[s;c]=-\frac{d_s}{2}\log(1-\rho^2)<\infty.
\]
Exact distribution matching and information saturation
correspond to $\Delta=0$.

Because of incomplete matching, $p_\theta$ and $q_f$ need not
be equal, and for the proof we define $p(z,z_c)=q_f(z,z_c)$.

\medskip
\noindent\textbf{Assumption 6
(Matched dimension and nondegenerate representation covariances).}
The target representation and signal have equal dimension,
$d:=d_z=d_s$.
The encoded representations have finite second moments and
positive-definite covariance matrices,
\begin{equation}
    V_z:=\operatorname{Cov}_{q_f}[z]\succ0,
    \qquad
    V_c:=\operatorname{Cov}_{q_f}[z_c]\succ0.
\label{eq:approx-covariances}
\end{equation}
The marginal $q_f(z)$ has a density with finite differential
entropy.

\begin{theorem}[Approximate Gaussian affine recovery]
\label{thm:approximate-readout}
Suppose Assumption~\ref{ass:markov} and
Assumptions~$2^\ast$, $3^\ast$, and 6 above hold.
Then there exist $A\in\mathbb R^{d\times d}$ and
$b\in\mathbb R^d$ such that
\begin{equation}
    \boxed{
    \left\langle\|s-Az-b\|^2\right\rangle_{p(s,z)}
    \leq \frac{e^{2\Delta}-1}{\rho^2}.}
\label{eq:approx-readout}
\end{equation}
Moreover, defining the normalized cross-covariance
\[
    C:=
    V_z^{-1/2}
    \operatorname{Cov}_{q_f}(z,z_c)
    V_c^{-1/2},
\]
the same affine readout satisfies
\begin{equation}
    \left\langle\|s-Az-b\|^2\right\rangle_{p(s,z)}
    \leq
    \frac{d\rho^2-\|C\|_F^2}
         {\rho^2(1-\rho^2)}.
\label{eq:approx-covariance-readout}
\end{equation}
% More precisely, if $\sigma_1^2,\ldots,\sigma_d^2$ are the
% eigenvalues of $CC^\top$, including zeros, then
% \begin{equation}
%     \left\langle\|s-Az-b\|^2\right\rangle_{p(s,z)}
%     \leq
%     \sum_{i=1}^d
%     \min\!\left\{
%         1,\,
%         \frac{\rho^2-\sigma_i^2}{\rho^2(1-\rho^2)}
%     \right\}.
% \label{eq:approx-spectral-readout}
% \end{equation}
\end{theorem}

\begin{proof}
The proof proceeds in four steps.
First, following the Hermite-decomposition approach of
\citet{klindt2026does}, we show that nonlinear dependence
of the representation on the signal reduces the variance
predictable from the true context.
Hermite polynomials are useful here because they are the
eigenfunctions of Gaussian conditional-averaging, simplifying
the problem because Gaussian prediction acts diagonally on
this basis.
Second, we use this reduction of the variance to show that
poor affine recovery necessarily increases the prediction
residual covariance, even when conditioning on the learned
context.
Third, we show that this residual covariance limits the
Gaussian predictive likelihood and hence the InfoLDM objective.
Finally, we combine these inequalities to show that a small
objective gap guarantees a small affine recovery error.

For the proof, define
\[
    m_z:=\langle z\rangle_{q_f(z)},
    \qquad
    m_c:=\langle z_c\rangle_{q_f(z_c)},
\]
and introduce the centered, whitened coordinates
\[
    \widetilde z:=V_z^{-1/2}(z-m_z),
    \qquad
    \widetilde z_c:=V_c^{-1/2}(z_c-m_c).
\]
Both have zero mean and identity covariance, and
\[
    C=\left\langle
        \widetilde z\widetilde z_c^\top
      \right\rangle.
\]
These invertible affine transformations preserve the
conditional independences in Assumption~\ref{ass:markov}.

\medskip
\noindent\emph{Step 1. Gaussian prediction suppresses
nonlinear components.}
Define
\[
    u:=\rho^{-1}\mu_\star(c).
\]
By Assumption~$2^\ast$, we may write
\[
    s=\rho u+\sqrt{1-\rho^2}\,\epsilon,
    \qquad
    \epsilon\sim\mathcal N(0,I_d),
    \qquad \epsilon\indep c.
\]
Since $s$ is standard Gaussian and $\epsilon$ is independent
of $u$, taking characteristic functions shows that
$u\sim\mathcal N(0,I_d)$.
Thus $(s,u)$ is jointly Gaussian with
$\operatorname{Cov}[s,u]=\rho I_d$.

Let
\[
    B:=\left\langle\widetilde z s^\top\right\rangle,
    \qquad
    m(s):=\langle\widetilde z\rangle_{p(\widetilde z\mid s)}.
\]
Following the decomposition into Hermite degrees in
\citet[Appendix~C]{klindt2026does}, expand
\begin{equation}
    m(s)=Bs+\sum_{k\geq2}h_k(s),
\label{eq:encoder-hermite-expansion}
\end{equation}
where $h_k$ contains the terms of total degree $k$ in
the orthonormal multivariate Hermite basis under
$p(s)=\mathcal N(0,I_d)$.
The constant term vanishes because
$\langle\widetilde z\rangle=0$.
Conditional Jensen's inequality and orthogonality of the
Hermite basis give
\begin{equation}
\begin{split}
    \left\langle\|m(s)\|^2\right\rangle_{p(s)}
    &=
    \|B\|_F^2+
    \sum_{k\geq2}
      \left\langle\|h_k(s)\|^2\right\rangle_{p(s)}\\
    &\leq
    \left\langle\|\widetilde z\|^2\right\rangle
    =d.
\end{split}
\label{eq:encoder-energy}
\end{equation}
The inequality allows for variation in $\widetilde z$ due to
private nuisance at fixed $s$.

For the Gaussian pair $(s,u)$, the Hermite
conditional-averaging identity gives
\[
    \langle h_k(s)\rangle_{p(s\mid u)}
    =\rho^k h_k(u).
\]
This is the spectral contraction used in
\citet[Appendices~A and~C]{klindt2026does}.
Private nuisance gives $\widetilde z\indep c\mid s$.
Since $p(s\mid c)$ depends on $c$ only through $u$, defining
\[
    g(c):=\langle\widetilde z\rangle_{p(\widetilde z\mid c)}
\]
yields
\[
\begin{split}
    g(c)
    &=\langle m(s)\rangle_{p(s\mid c)}\\
    &=\rho Bu+\sum_{k\geq2}\rho^k h_k(u).
\end{split}
\]
Consequently,
\begin{equation}
\begin{split}
    \left\langle\|g(c)\|^2\right\rangle_{p(c)}
    &=
    \rho^2\|B\|_F^2+
    \sum_{k\geq2}\rho^{2k}
      \left\langle\|h_k(u)\|^2\right\rangle_{p(u)}\\
    &\leq
    \rho^2\|B\|_F^2+
    \rho^4\bigl(d-\|B\|_F^2\bigr).
\end{split}
\label{eq:gaussian-predictable-energy}
\end{equation}
Here $p(s)$ and $p(u)$ are both standard Gaussian, so
the squared norms of the corresponding Hermite terms agree.

Applying the same argument to any scalar projection
$v^\top\widetilde z$ gives
\begin{equation}
\begin{split}
    \left\langle(v^\top g(c))^2\right\rangle_{p(c)}
    &\leq
    \rho^2\|B^\top v\|^2
    +\rho^4\bigl(\|v\|^2-\|B^\top v\|^2\bigr)\\
    &\leq\rho^2\|v\|^2.
\end{split}
\label{eq:directional-contraction}
\end{equation}
Equivalently,
\begin{equation}
    \left\langle g(c)g(c)^\top\right\rangle_{p(c)}
    \preceq
    \rho^4 I_d+\rho^2(1-\rho^2)BB^\top
    \preceq\rho^2 I_d.
\label{eq:matrix-contraction}
\end{equation}

\medskip
\noindent\emph{Step 2. Conditional prediction error
controls recovery.}
In this step we establish a connection between the
normalized cross-covariance of the learned condition and
representation, and the average affine readout error.

Since $s$ and $\widetilde z$ are centered and whitened,
the optimal affine readout in these coordinates is
$B^\top\widetilde z$.
In the original coordinates, this corresponds to
\[
    A=B^\top V_z^{-1/2},
    \qquad b=-Am_z.
\]
The average affine readout error is therefore
\begin{equation}
\begin{split}
    E_{\mathrm{rec}}
    &:=\inf_{A',b'}
       \left\langle\|s-A'z-b'\|^2\right\rangle_{p(s,z)}\\
    &=\left\langle\|s-B^\top\widetilde z\|^2\right\rangle
     =d-\|B\|_F^2.
\end{split}
\label{eq:optimal-linear-error}
\end{equation}

Define the learned context-dependent conditional mean and
its residual covariance,
\begin{equation}
\begin{split}
    \bar z(\widetilde z_c)
    &:=\langle\widetilde z\rangle_{p(\widetilde z\mid\widetilde z_c)},\\
    \Omega
    &:=
    \left\langle
      (\widetilde z-\bar z(\widetilde z_c))
      (\widetilde z-\bar z(\widetilde z_c))^\top
    \right\rangle.
\end{split}
\label{eq:actual-prediction-residual}
\end{equation}
These quantities are defined by the encoded distribution.
They need not equal the learned predictor mean or covariance
in these coordinates.

Assumption~\ref{ass:markov} gives
$\widetilde z\indep\widetilde z_c\mid c$, so
\[
    \bar z(\widetilde z_c)
    =\langle g(c)\rangle_{p(c\mid\widetilde z_c)}.
\]
Conditional Jensen's inequality therefore implies
\[
    \left\langle
      \bar z(\widetilde z_c)\bar z(\widetilde z_c)^\top
    \right\rangle
    \preceq
    \left\langle g(c)g(c)^\top\right\rangle_{p(c)}.
\]
By conditional-mean orthogonality,
\[
    \Omega
    =I_d-
      \left\langle
        \bar z(\widetilde z_c)\bar z(\widetilde z_c)^\top
      \right\rangle.
\]
Together with \Eqref{eq:directional-contraction},
this proves
\begin{equation}
    \Omega\succeq(1-\rho^2)I_d\succ0.
\label{eq:residual-covariance-lower-bound}
\end{equation}
Taking traces and using
\Eqref{eq:gaussian-predictable-energy} also gives
\begin{equation}
\begin{split}
    \operatorname{tr}\Omega
    &\geq d-
      \left\langle\|g(c)\|^2\right\rangle_{p(c)}\\
    &\geq d(1-\rho^2)
      +\rho^2(1-\rho^2)\bigl(d-\|B\|_F^2\bigr).
\end{split}
\label{eq:residual-trace-bound}
\end{equation}
Using \Eqref{eq:optimal-linear-error},
\begin{equation}
    \operatorname{tr}\Omega-d(1-\rho^2)
    \geq\rho^2(1-\rho^2)E_{\mathrm{rec}}.
\label{eq:residual-recovery-bound}
\end{equation}

To obtain the cross-covariance bound, we first relate $C$
to the part of $\widetilde z$ predictable from the true
context, $g(c)$.
Since $\widetilde z\indep\widetilde z_c\mid c$,
\[
    C
    =\left\langle
        \widetilde z\widetilde z_c^\top
      \right\rangle
    =\left\langle g(c)\widetilde z_c^\top\right\rangle.
\]
Because $\widetilde z_c$ is centered and whitened,
$C\widetilde z_c$ is the best linear prediction of $g(c)$
from $\widetilde z_c$.
Expanding its residual covariance gives
\[
\begin{split}
    0
    &\preceq
    \left\langle
      (g(c)-C\widetilde z_c)
      (g(c)-C\widetilde z_c)^\top
    \right\rangle\\
    &=
    \left\langle g(c)g(c)^\top\right\rangle_{p(c)}
    -CC^\top.
\end{split}
\]
Hence, by \Eqref{eq:matrix-contraction},
\[
    CC^\top
    \preceq
    \rho^4 I_d+\rho^2(1-\rho^2)BB^\top.
\]
Taking traces and substituting
$E_{\mathrm{rec}}=d-\|B\|_F^2$ gives
\[
    \|C\|_F^2
    \leq d\rho^2-\rho^2(1-\rho^2)E_{\mathrm{rec}}.
\]
Rearranging proves \Eqref{eq:approx-covariance-readout}.

% For the sharper bound, define $D:=I_d-BB^\top$.
% Conditional Jensen's inequality gives $BB^\top\preceq I_d$,
% so
% \[
%     0\preceq D\preceq I_d,
%     \qquad
%     D\preceq
%     \frac{\rho^2 I_d-CC^\top}{\rho^2(1-\rho^2)}.
% \]
% Let $v_1,\ldots,v_d$ be an orthonormal eigenbasis of
% $CC^\top$, with corresponding eigenvalues
% $\sigma_1^2,\ldots,\sigma_d^2$.
% Then
% \[
%     v_i^\top Dv_i
%     \leq
%     \min\!\left\{
%         1,\,
%         \frac{\rho^2-\sigma_i^2}{\rho^2(1-\rho^2)}
%     \right\}.
% \]
% Summing over $i$ and using
% $\operatorname{tr}D=E_{\mathrm{rec}}$
% proves \Eqref{eq:approx-spectral-readout}.

\medskip
\noindent\emph{Step 3. Conditional residual covariance
bounds the objective.}
The InfoLDM objective satisfies
\[
    \mathcal F
    \leq H_{q_f}[z]
       +\left\langle
          \log p_\theta(z\mid z_c)
        \right\rangle_{q_f(z,z_c)}.
\]
Equality holds if $p_\theta(z_c)=q_f(z_c)$; otherwise
the right-hand side exceeds $\mathcal F$ by
$\mathrm{KL}[q_f(z_c)\|p_\theta(z_c)]$.

In the transformed coordinates, the predictor is Gaussian
with mean and covariance
\[
\begin{split}
    \widetilde\mu_\theta(\widetilde z_c)
    &:=
    V_z^{-1/2}
    \left(
      \mu_\theta(m_c+V_c^{1/2}\widetilde z_c)-m_z
    \right),\\
    \widetilde\Sigma_\theta
    &:=
    V_z^{-1/2}\Sigma_\theta V_z^{-1/2}.
\end{split}
\]
The entropy and log-likelihood changes cancel:
\[
\begin{split}
    H[z]
    &=H[\widetilde z]+\frac12\log\det V_z,\\
    \left\langle\log p_\theta(z\mid z_c)\right\rangle
    &=
    \left\langle
      \log\widetilde p_\theta
      (\widetilde z\mid\widetilde z_c)
    \right\rangle
    -\frac12\log\det V_z.
\end{split}
\]
Consequently,
\[
    \mathcal F
    \leq H[\widetilde z]
       +\left\langle
          \log\widetilde p_\theta
          (\widetilde z\mid\widetilde z_c)
        \right\rangle.
\]
We now find an upper bound for this expression.

For any learned mean function
$\widetilde\mu_\theta(\widetilde z_c)$,
conditional-mean orthogonality gives
\begin{equation}
\begin{split}
    &\left\langle
       (\widetilde z-\widetilde\mu_\theta)
       (\widetilde z-\widetilde\mu_\theta)^\top
     \right\rangle\\
    &\quad=
    \Omega+
    \left\langle
      (\bar z-\widetilde\mu_\theta)
      (\bar z-\widetilde\mu_\theta)^\top
    \right\rangle
    \succeq\Omega,
\end{split}
\label{eq:nonlinear-prediction-residual}
\end{equation}
where both mean functions are evaluated at $\widetilde z_c$.
Consequently, the Gaussian predictor likelihood satisfies
\[
\begin{split}
    \left\langle
      \log\widetilde p_\theta
      (\widetilde z\mid\widetilde z_c)
    \right\rangle
    \leq{}&
    -\frac d2\log(2\pi)
    -\frac12\log\det\widetilde\Sigma_\theta\\
    &-\frac12
      \operatorname{tr}(\widetilde\Sigma_\theta^{-1}\Omega).
\end{split}
\]
Maximizing this upper bound over positive-definite
$\widetilde\Sigma_\theta$ gives
$\widetilde\Sigma_\theta=\Omega$ and value
\[
    -\frac d2\log(2\pi e)-\frac12\log\det\Omega.
\]
Using the Gaussian entropy upper bound
$H[\widetilde z]\leq\frac d2\log(2\pi e)$, we obtain
\begin{equation}
    \mathcal F\leq-\frac12\log\det\Omega.
\label{eq:gaussian-objective-upper-bound}
\end{equation}

\medskip
\noindent\emph{Step 4. The objective gap bounds recovery.}
We now combine the objective bound from Step~3 with
the recovery bound from Step~2.
Under the Gaussian signal distributions,
\[
    I[s;c]=H_p[s]-H_p[s\mid c]
          =-\frac d2\log(1-\rho^2).
\]
Substituting this expression and
$\mathcal F\leq-\frac12\log\det\Omega$ into the
definition of $\Delta$ gives
\[
\begin{split}
    \Delta
    &=I[s;c]-\mathcal F\\
    &\geq
      -\frac d2\log(1-\rho^2)
      +\frac12\log\det\Omega\\
    &=\frac12\log
      \frac{\det\Omega}{(1-\rho^2)^d}.
\end{split}
\]
Thus a small objective gap constrains the determinant
of the residual covariance.

To relate this determinant to the excess residual variance
appearing in \Eqref{eq:residual-recovery-bound}, let
$\lambda_1,\ldots,\lambda_d$ be the eigenvalues of $\Omega$.
By \Eqref{eq:residual-covariance-lower-bound}, we can write
\[
    \lambda_i=(1-\rho^2)(1+a_i),
    \qquad
    a_i:=\frac{\lambda_i-(1-\rho^2)}{1-\rho^2}\geq0.
\]
Here $a_i$ measures the relative excess residual variance
in the corresponding eigenvector direction.
Exponentiating the preceding gap bound gives
\[
    e^{2\Delta}\geq
    \frac{\det\Omega}{(1-\rho^2)^d}
    =\prod_{i=1}^d(1+a_i).
\]
Expanding this product yields $1+\sum_i a_i$ together
with products of two or more $a_i$.
All these additional terms are nonnegative, so
\[
\begin{split}
    e^{2\Delta}-1
    &\geq\sum_{i=1}^d a_i\\
    &=\frac{\operatorname{tr}\Omega-d(1-\rho^2)}
            {1-\rho^2}.
\end{split}
\]
The objective gap therefore bounds the excess residual
variance from above.
Step~2 established that this same excess variance is
bounded from below by
$\rho^2(1-\rho^2)E_{\mathrm{rec}}$.
Combining the two inequalities gives
\[
    e^{2\Delta}-1
    \geq
    \frac{\operatorname{tr}\Omega-d(1-\rho^2)}
         {1-\rho^2}
    \geq \rho^2E_{\mathrm{rec}}.
\]
Dividing by $\rho^2$ proves \Eqref{eq:approx-readout}.
\end{proof}

\newpage

\subsection{Lean formalization}

To verify the main proof (Theorem \ref{thm:exponential-readout}) we translated it to compiling Lean code using generative AI. 
The declaration \nolinkurl{statisticReadout_of_exponentialFamily_nuisanceLDM} formalizes Theorem~\ref{thm:exponential-readout} after the mutual-information argument.
Below is the header of the main function which includes the main assumptions and the conclusion.
The line numbers refer to the Lean source code, the full declaration and proof can be found in the source code at \nolinkurl{Lean/ExponentialFamilyReadout.lean}.

\begin{Verbatim}
758 theorem statisticReadout_of_exponentialFamily_nuisanceLDM
759     {Ω H V X Y : Type*}
760     [MeasurableSpace Ω] [StandardBorelSpace Ω]
761     [MeasurableSpace H] [StandardBorelSpace H] [Nonempty H]
762     [MeasurableSpace V] [StandardBorelSpace V] [Nonempty V]
763     [MeasurableSpace X] [StandardBorelSpace X] [Nonempty X]
764     [MeasurableSpace Y] [StandardBorelSpace Y] [Nonempty Y]
765     (P : Measure Ω) [IsProbabilityMeasure P]
766     (history : Ω → H) (view : Ω → V)
767     (signal : Ω → X) (code : Ω → Y)
768     (hhistory : Measurable history) (hview : Measurable view)
769     (hsignal : Measurable signal) (hcode : Measurable code)
770     (signal_history_given_code : signal ⟂ᵢ[code, hcode; P] history)
771     (code_history_given_signal : code ⟂ᵢ[signal, hsignal; P] history)
772     (code_history_given_view : code ⟂ᵢ[view, hview; P] history)
773     (code_view_given_history : code ⟂ᵢ[history, hhistory; P] view)
774     (sourceCarrier : Measure X) (codeCarrier : Measure Y)
775     (sourceStatistic : X → EuclideanSpace ℝ ι)
776     (codeStatistic : Y → EuclideanSpace ℝ κ)
777     (hsourceStatistic : Measurable sourceStatistic)
778     (hcodeStatistic : Measurable codeStatistic)
779     (sourcePartition : EuclideanSpace ℝ ι → ℝ)
780     (codePartition : EuclideanSpace ℝ κ → ℝ)
781     (sourceParameter : H → EuclideanSpace ℝ ι)
782     (learnedParameter : V → EuclideanSpace ℝ κ)
783     (hsourceConditional :
784       condDistrib signal history P =ᵐ[P.map history]
785         fun c ↦ statisticNaturalExpMeasure sourceCarrier sourceStatistic
786           sourcePartition (sourceParameter c))
787     (hlearnedConditional :
788       condDistrib code view P =ᵐ[P.map view]
789         fun v ↦ statisticNaturalExpMeasure codeCarrier codeStatistic
790           codePartition (learnedParameter v))
791     (hhistoryDiversity :
792       ∀ good : Set H, (∀ᵐ c ∂P.map history, c ∈ good) →
793         ∃ (c₀ : H) (c : ι → H),
794           c₀ ∈ good ∧
795           (∀ i, c i ∈ good) ∧
796           Submodule.span ℝ
797             (Set.range fun i ↦ sourceParameter (c i) - sourceParameter c₀) = ⊤)
798     (hsourceNoAffineRedundancy :
799       ∀ (w : EuclideanSpace ℝ ι) (a : ℝ),
800         (fun x ↦ inner ℝ w (sourceStatistic x)) =ᵐ[sourceCarrier]
801             (fun _ ↦ a) →
802           w = 0) :
803     ∃ (A : EuclideanSpace ℝ κ →L[ℝ] EuclideanSpace ℝ ι)
804         (b : EuclideanSpace ℝ ι),
805       (fun ω ↦ sourceStatistic (signal ω)) =ᵐ[P]
806           (fun ω ↦ A (codeStatistic (code ω)) + b) ∧
807       Function.Surjective A ∧
808       Module.finrank ℝ (EuclideanSpace ℝ ι) ≤
809         Module.finrank ℝ (EuclideanSpace ℝ κ) ∧
810       (Function.Injective sourceStatistic →
811         ∃ decoder : Y → X, Measurable decoder ∧
812           signal =ᵐ[P] fun ω ↦ decoder (code ω)) ∧
813       (Module.finrank ℝ (EuclideanSpace ℝ κ) =
814           Module.finrank ℝ (EuclideanSpace ℝ ι) →
815         Function.Bijective A) := by
\end{Verbatim}

This declaration reads in blocks as follows.

\paragraph{Lines 759--769: probability space and variables.}
\texttt{Omega} is the underlying probability space, and \texttt{H}, \texttt{V}, \texttt{X}, and \texttt{Y} are the state spaces of $c,z_c,s,z$.
The finite types $\iota$ and $\kappa$ index the coordinates of $\tau_\star$ and $\tau_z$, so their \texttt{finrank}s are $r_s$ and $r_z$.

\enlargethispage{2\baselineskip}
\paragraph{Lines 770--773: the four conditional independences.}
In the order shown, these are
\begin{equation}
    s\indep c\mid z,
    \qquad
    z\indep c\mid s,
    \qquad
    z\indep c\mid z_c,
    \qquad
    z\indep z_c\mid c.
    \label{eq:lean-four-independences}
\end{equation}
The second and fourth are the Markov consequences of Assumption~\ref{ass:markov}.
The first and third follow in the paper from
\begin{equation}
    I[s;c]-I[z;z_c]=I[s;c\mid z]+I[z;c\mid z_c]=0.
    \label{eq:lean-mi-boundary}
\end{equation}
Thus Lean starts at the conditional-independence consequences of MI saturation, rather than formalizing mutual information and the optimization problem.

\paragraph{Lines 774--790: the two exponential families.}
The names \texttt{sourceCarrier}, \texttt{sourceStatistic}, \texttt{sourcePartition}, and \texttt{sourceParameter} correspond to $h_s,\tau_\star,\psi,\alpha$.
Similarly, $h_z,\tau_z,\phi,\beta$ refer to the learned predictive distribution.
The two hypotheses \texttt{hsourceConditional} and \texttt{hlearnedConditional} are the conditional families
\begin{align}
    p(s\mid c)
    &=h_s(s)\exp\!\left(\alpha(c)^\top\tau_\star(s)-\psi(\alpha(c))\right), \\
    p(z\mid z_c)
    &=h_z(z)\exp\!\left(\beta(z_c)^\top\tau_z(z)-\phi(\beta(z_c))\right),
    \label{eq:lean-family-correspondence}
\end{align}
where \texttt{statisticNaturalExpMeasure} is the Lean implementation of the exponential-family distribution in
\Eqref{eq:source-family} and \Eqref{eq:learned-model-family}.

\paragraph{Lines 791--802: diversity and nonredundancy.}
The hypothesis \texttt{hconditionDiversity} is Assumption~\ref{ass:condition-diversity} written without coordinates: after removing any zero-probability set of conditions, it selects $c_0$ and a family $(c_i)$ whose differences $\alpha(c_i)-\alpha(c_0)$ span the whole source-statistic space.
The hypothesis \texttt{hsourceNoAffineRedundancy} is Assumption~\ref{ass:source-nondegenerate}: if $\langle w,\tau_\star(s)\rangle$ is almost surely constant on the carrier, then $w=0$.

\paragraph{Lines 803--815: conclusion.}
Lean returns a continuous linear map $A:\R^{r_z}\to\R^{r_s}$ and an offset $b$ with
\begin{equation}
    \tau_\star(s)=A\tau_z(z)+b
    \qquad\text{almost surely}.
    \label{eq:lean-readout-conclusion}
\end{equation}
\texttt{Function.Surjective A} is the full-row-rank conclusion $\rank(A)=r_s$, and the \texttt{finrank} inequality is $r_s\le r_z$.
The last two lines add the two corollaries in Theorem~\ref{thm:exponential-readout}: an injective $\tau_\star$ gives a measurable decoder $z\mapsto s$, and equal statistic dimensions make $A$ bijective.
The proof first establishes the likelihood-ratio identity at a reference condition and then extends it to the full model.

\newpage

\subsection{Simulation details}

\subsubsection{Stein entropy gradient estimator}
\label{app:stein}

We use the Stein gradient estimator to optimize the entropy of the learned representations \citep{li2017gradient,wen2024mveb}.
The entropy gradient can be computed as
\[
\nabla_f H_{q_f}[z] = 
-
\left\langle
\left(\frac{\partial f(x)}{\partial f}\right)^{\!\top}
\left.\nabla_z\log q_f(z)\right|_{z=f(x)}
\right\rangle_{p(x)}.
\]
% The score function $\mathcal{s}(z):=\nabla_z \log q_f(z)$ can be estimated with a black-bock estimator.
Stein’s identity allows to estimate the unknown score $\mathfrak{s}(z):=\nabla_z \log q_f(z)$ by solving a regularized kernel-based least-squares problem. Substituting these score estimates into the expression above yields an approximation to the entropy gradient without explicitly estimating \(q_f\).

For representations $z_i=f(x_i)$ and estimated scores
$\hat{\mathfrak{s}}_i$, we implement entropy maximization by minimizing
the surrogate
\begin{equation}
    \mathcal{L}_{\mathrm{Stein}}
    = \frac{1}{B}\sum_{i=1}^{B}
      \operatorname{stopgrad}(\hat{\mathfrak{s}}_i)^\top z_i ,
\end{equation}
where $B$ is the minibatch size.
Our implementation uses a Gaussian RBF kernel with bandwidth equal to the median pairwise distance within the minibatch and ridge regularization of $10^{-2}$. 

\subsubsection{Setups}
\label{app:sim}

\paragraph{Numerical evaluation.}

We generated $1000000$ training and $10000$ test sequences of length five. The fixed observation map consisted of three affine layers with leaky-ReLU activations and mapped the concatenated signal and nuisance variables to $100$-dimensional observations. The encoder was a five-layer \ac{MLP} with $200$ hidden units per layer and ReLU activations. Its output dimension matched the signal dimension, which was varied over $d_S\in\{5,10,15,20\}$. A single-layer \ac{LSTM} with $2d_S$ hidden units summarized the observation history, and a linear prediction head predicted the next latent state.

Models were trained for $20$ epochs using Adam and a linearly decaying learning rate. Batch size was $256$ for the kNN, log-determinant, and Stein estimators, and $4096$ for KDE. 
In this and the following experiments, entropy estimation was implemented as described by \citet{mikulasch2026understanding}. 
% We found that KDE performed poorly in this experiment and increased batchsize to $4096$ \citep[similar to ][]{kugelgen_self-supervised_2022} and the number of samples to $5\times10^6$ to improve performance. 

\paragraph{Causal3DIdent.}

We based our experiment on the data of \citet{kugelgen_self-supervised_2022} but resized images to $224\times224$ pixels. We used a ResNet-18 encoder producing a $10$-dimensional representation, followed by a single-layer \ac{LSTM} with $64$ hidden units and a linear prediction head. Models were trained for $25$ epochs with Adam. Batch sizes were $256$ for the kNN and KDE estimators and $64$ for the log-determinant estimator.

\paragraph{MuJoCo Hopper.}

We generated $100000$ training and $1000$ test rollouts from the DM-Control Hopper environment. Each sequence contained eight $128\times128$ rendered frames, separated by three control steps. Actions were produced by a pretrained controller \citep{nikulin2025latent} and perturbed with Gaussian exploration noise of standard deviation $0.1$. Additional unobserved actuator noise with standard deviation $0.5$ was applied before executing each action, making the controlled dynamics stochastic.
Note that adding Gaussian noise every simulation step, not every render step, results in a system that does not fully follow the assumptions of Theorem \ref{thm:exponential-readout}.
At every rendered timestep, independent nuisance variables modified the body and accent colors, camera position and field of view, illumination, and a procedurally generated colored background. The physical signal used for evaluation consisted of seven position variables and seven velocity variables.

Images were encoded by a ResNet-18 into a $16$-dimensional latent state. A single-layer \ac{LSTM} with $64$ hidden units summarized the latent history. The three controller commands between consecutive frames were encoded by a GRU with $16$ hidden units and supplied, together with the \ac{LSTM} state, to a one-hidden-layer prediction \ac{MLP}. For image decoding we used 4-layer upsampling CNNs trained with L2 loss. With attached decoder loss we approximately balanced latent prediction and image prediction loss via a Lagrange multiplier. For the observation prediction model we used the exact same setup without latent prediction loss and the decoder attached to the prediction module. Models were trained for $40$ epochs using Adam. Batch sizes were $128$, $512$, $32$, and $128$ for the kNN, KDE, log-determinant estimators, and next observation prediction, respectively.

\newpage
\raggedbottom

\subsection{Additional results}

\subsubsection{Generalized Gaussian Ablation}

\begin{table}[h!]
\caption{Linear $R^2$ for generalized Gaussian dynamics noise with 5D signals and 5D nuisance. The experiment uses the same setup as in Table \ref{tab:numerical}. $\beta$ is the exponent in the generalized Gaussian. Entries show mean $\pm$ standard deviation for 5 seeds. Similar to previous work \citep{mikulasch2026understanding} we find robustness to the shape of noise.}
\label{tab:generalized-gaussian}
    \vspace{3mm}
\centering
\begin{tabular}{lcccc}
\toprule
Estimator & $\beta = 8$ & $\beta = 4$ & $\beta = 1$ & $\beta = 0.7$ \\
\midrule
KDE & $0.71 \pm 0.16$ & $0.60 \pm 0.32$ & $0.71 \pm 0.28$ & $0.75 \pm 0.19$ \\
kNN & $\mathbf{0.99 \pm 0.00}$ & $\mathbf{0.99 \pm 0.00}$ & $\mathbf{0.99 \pm 0.00}$ & $\mathbf{0.98 \pm 0.00}$ \\
logdet & $\mathbf{1.00 \pm 0.00}$ & $\mathbf{0.96 \pm 0.07}$ & $\mathbf{0.95 \pm 0.08}$ & $\mathbf{0.94 \pm 0.08}$ \\
Stein & $\mathbf{1.00 \pm 0.00}$ & $\mathbf{1.00 \pm 0.00}$ & $\mathbf{0.99 \pm 0.00}$ & $\mathbf{0.99 \pm 0.00}$ \\
\bottomrule
\end{tabular}
\end{table}

\subsubsection{Weakly predictable nuisance ablation}

\begin{table}[H]
\caption{
Robustness to predictable nuisance (soft violation of Assumption 1).
We use the same setup as in Table \ref{tab:numerical} with 5D signals and 5D nuisance.
Two of the nuisance coordinates follow $n_t=\lambda Q_n n_{t-1}+\sqrt{1-\lambda^2}\,\epsilon_t$, where $Q_n$ is a nuisance rotation matrix and the $\epsilon_i$ are independent standard-normal noise.
The remaining three nuisance coordinates remain private, and signal generation is unchanged. Thus $\lambda=0$ is the private-nuisance control.
Scores are averaged over the five signal coordinates (top) or the two correlated nuisance coordinates (bottom).
Since the model latent variables are only 5D, the model has to choose between recovering the true signal, or the nuisance.
Results show that models prefer to recover the more strongly correlated signal coordinates until nuisance becomes strongly predictable ($\lambda=0.8$), except for KDE, which partly encodes nuisance already earlier.
}
\vspace{3mm}
\label{tab:predictable-nuisance}
\centering
\small
\setlength{\tabcolsep}{5pt}
\begin{tabular}{lcccc}
\toprule
Estimator & $\lambda=0$ & $\lambda=0.2$ & $\lambda=0.5$ & $\lambda=0.8$ \\
\midrule
\multicolumn{5}{c}{Affine signal recovery } \\
\midrule
KDE & $0.876 \pm 0.142$ & $0.764 \pm 0.134$ & $0.765 \pm 0.179$ & $0.596 \pm 0.003$ \\
kNN & $\mathbf{0.989 \pm 0.000}$ & $\mathbf{0.989 \pm 0.000}$ & $\mathbf{0.949 \pm 0.080}$ & $0.832 \pm 0.147$ \\
logdet & $\mathbf{0.998 \pm 0.000}$ & $\mathbf{0.998 \pm 0.000}$ & $\mathbf{0.998 \pm 0.000}$ & $0.878 \pm 0.159$ \\
Stein & $\mathbf{0.995 \pm 0.000}$ & $\mathbf{0.995 \pm 0.000}$ & $\mathbf{0.916 \pm 0.097}$ & $0.837 \pm 0.148$ \\
\midrule
\multicolumn{5}{c}{Predictable-nuisance recovery} \\
\midrule
KDE & $0.005 \pm 0.007$ & $0.061 \pm 0.109$ & $0.203 \pm 0.191$ & $0.602 \pm 0.194$ \\
kNN & $-0.000 \pm 0.000$ & $-0.000 \pm 0.000$ & $-0.000 \pm 0.000$ & $0.388 \pm 0.365$ \\
logdet & $-0.000 \pm 0.000$ & $-0.000 \pm 0.000$ & $-0.000 \pm 0.000$ & $0.293 \pm 0.393$ \\
Stein & $-0.000 \pm 0.000$ & $-0.000 \pm 0.000$ & $-0.000 \pm 0.000$ & $0.394 \pm 0.370$ \\
\bottomrule
\end{tabular}
\end{table}

\subsubsection{Condition diversity ablation}

\begin{table}[H]
\centering
\caption{Robustness to reduced signal predictability (soft violation of Assumption 4). 
We use the same setup as in Table \ref{tab:numerical} with 5D signals and 5D nuisance and 5-step sequences.
Signals evolve as $s_{t+1,i}=\rho_i s_{t,i}+\epsilon_{t,i}$, with $s_0\sim\mathcal{N}(0,I)$ and independent $\epsilon_{t,i}\sim\mathcal{N}(0,1-\rho_i^2)$. 
We fix $\rho_1=\cdots=\rho_4=0.9$ as before and vary $\rho_5$ to decrease its predictability. At $\rho_5=0$ the signal violates the condition diversity assumption, at low but non-zero $\rho_5$ it is softly violated.
We report recovery of the full signal $s$, of the less predictable dimension $s_5$, and of the consistently strongly predictable dimensions $s_1,\dots,s_4$.
Results demonstrate that reduced predictability leads to a reduced signal recovery.
This is expected also from the approximate recovery Theorem \ref{thm:approximate-readout}, which shows that a reduction in signal predictability increases the recovery error.
Consequently, most methods show recovery that decays gracefully with decreasing predictability.
The logdet estimator experiences an unexpected catastrophic collapse of the dimension already at $\rho_5=0.5$.
Strongly predictable dimensions are consistently recovered with the same fidelity.
}
\label{tab:condition-diversity-diagonal}
\vspace{3mm}
\small
\setlength{\tabcolsep}{4pt}
\begin{tabular}{lccccc}
\toprule
Estimator & $\rho_5=0.9$ & $\rho_5=0.5$ & $\rho_5=0.2$ & $\rho_5=0.05$ & $\rho_5=0$ \\
\midrule
\multicolumn{6}{c}{Less predictable signal dimensions recovery} \\
\midrule
KDE & $\mathbf{0.942 \pm 0.101}$ & $0.603 \pm 0.394$ & $0.265 \pm 0.304$ & $0.093 \pm 0.074$ & $0.099 \pm 0.115$ \\
kNN & $\mathbf{0.989 \pm 0.000}$ & $0.886 \pm 0.003$ & $0.496 \pm 0.198$ & $0.035 \pm 0.028$ & $0.032 \pm 0.060$ \\
logdet & $\mathbf{0.998 \pm 0.000}$ & $-0.002 \pm 0.000$ & $-0.001 \pm 0.000$ & $-0.001 \pm 0.000$ & $-0.001 \pm 0.000$ \\
Stein & $\mathbf{0.996 \pm 0.000}$ & $\mathbf{0.932 \pm 0.001}$ & $0.651 \pm 0.007$ & $0.108 \pm 0.116$ & $0.051 \pm 0.055$ \\
\midrule
\multicolumn{6}{c}{Strongly predictable signal dimensions recovery} \\
\midrule
KDE & $\mathbf{0.957 \pm 0.074}$ & $\mathbf{0.994 \pm 0.001}$ & $\mathbf{0.995 \pm 0.000}$ & $\mathbf{0.995 \pm 0.000}$ & $\mathbf{0.995 \pm 0.000}$ \\
kNN & $\mathbf{0.989 \pm 0.000}$ & $\mathbf{0.991 \pm 0.000}$ & $\mathbf{0.990 \pm 0.001}$ & $\mathbf{0.987 \pm 0.002}$ & $\mathbf{0.988 \pm 0.003}$ \\
logdet & $\mathbf{0.998 \pm 0.000}$ & $\mathbf{0.998 \pm 0.000}$ & $\mathbf{0.998 \pm 0.000}$ & $\mathbf{0.998 \pm 0.000}$ & $\mathbf{0.998 \pm 0.000}$ \\
Stein & $\mathbf{0.996 \pm 0.000}$ & $\mathbf{0.996 \pm 0.000}$ & $\mathbf{0.996 \pm 0.000}$ & $\mathbf{0.996 \pm 0.000}$ & $\mathbf{0.996 \pm 0.000}$ \\
\bottomrule
\end{tabular}
\end{table}

\subsubsection{Component-wise recovery under heteroskedasticity}

\begin{table}[H]
\caption{Component-wise recovery under Gaussian predictions with changing diagonal covariance. 
We use the same setup as in Table \ref{tab:numerical} with 5D signals and 5D nuisance and 5-step sequences over 5 runs.
The signal has mean $s_t=0.9Q_s s_{t-1}$ and variance $v_i(s_t)=0.19\exp(\alpha\tanh(s_{t,i}))$. 
Additionally to the mean predictor in the previous setup we also train a diagonal variance predictor based on the same LSTM+MLP architecture.
According to Corollary \ref{cor:gaussian-permutation} this should allow readout of the signal via permutation if $\alpha>0$. 
Permuted component recovery is learned by using the Hungarian algorithm and a component-wise linear regression.
Additionally, we test $\alpha=0$, which is the same fixed variance setup as before. 
In practice we find that for the more complex model architecture entropy estimation presents a bottleneck:
Only Stein entropy gradient estimation reliably achieves component-wise recovery when $\alpha>0$.
Other estimators do not even achieve affine recovery for the baseline setup anymore, presenting a clear learning failure.
Gray boxes indicate conditions which should result in successful recovery.
}
\label{tab:component-recovery-diagonal}
\vspace{3mm}
\centering
\small
\begin{tabular}{lcc}
\toprule
Estimator & $\alpha=0$ & $\alpha=1$ \\
\midrule
\multicolumn{3}{c}{Permuted component recovery} \\
\midrule
KDE & $0.061 \pm 0.048$  & \cellcolor{gray!15} $0.126 \pm 0.041$  \\
kNN & $0.313 \pm 0.062$  & \cellcolor{gray!15} $0.382 \pm 0.052$  \\
logdet & $0.073 \pm 0.010$  & \cellcolor{gray!15} $0.083 \pm 0.019$  \\
Stein & $0.545 \pm 0.052$  & \cellcolor{gray!15} $\mathbf{0.992 \pm 0.000}$  \\
\midrule
\multicolumn{3}{c}{Affine recovery} \\
\midrule
KDE & \cellcolor{gray!15} $0.135 \pm 0.075$  & \cellcolor{gray!15} $0.224 \pm 0.073$  \\
kNN & \cellcolor{gray!15} $0.548 \pm 0.072$  & \cellcolor{gray!15} $0.572 \pm 0.014$  \\
logdet & \cellcolor{gray!15} $0.163 \pm 0.024$  & \cellcolor{gray!15} $0.151 \pm 0.029$  \\
Stein & \cellcolor{gray!15} $\mathbf{0.995 \pm 0.001}$  & \cellcolor{gray!15} $\mathbf{0.992 \pm 0.000}$  \\
\bottomrule
\end{tabular}
\end{table}

\subsubsection{Overcomplete representations}

\begin{table}[H]
\caption{Robustness of linear recovery under overparametrized representations. 
We use the same setup as in Table \ref{tab:numerical} with 5D signals and 5D nuisance and 5-step sequences over 5 runs, changing model latent dimensionality $d_z$.
LSTM dimensionality is fixed at 128.
$d_z=4$ is a control condition, showing that full recovery fails.
For higher latent dimension recovery is generally good.
Additionally we test Corollary \ref{cor:gaussian-nuisance}, checking if only the dimensions used for the best linear readout contain signal information.
We note that the requirement of $d_z\leq d_s+d_n$ of Corollary \ref{cor:gaussian-nuisance} is not fulfilled for conditions with $d_z>10$, which we therefore consider empirical tests beyond the guarantees of our theory.
We project the representation to obtain the residual representation $z_\text{res}=(I-A^+A)z$, where $A$ is the readout matrix, and learn another linear readout. Results show that residual dimensions contain little signal information for Stein and logdet entropy estimators, while the other estimators appear not to be accurate enough to conform to the theoretical results.
Gray boxes indicate conditions which should result in successful recovery.
}
\label{tab:randstep-latent-dimension}
\vspace{3mm}
\tiny
\centering
\begin{tabular}{llccccc}
\toprule
Estimator & Probe & $d_z=4$ & $d_z=8$ & $d_z=16$ & $d_z=32$ & $d_z=128$ \\
\midrule
KDE & Full dimension & $0.593 \pm 0.155$ & \cellcolor{gray!15} $0.854 \pm 0.095$ & \cellcolor{gray!15} $\mathbf{0.983 \pm 0.006}$ & \cellcolor{gray!15} $\mathbf{0.988 \pm 0.001}$ & \cellcolor{gray!15} $\mathbf{0.987 \pm 0.001}$ \\
KDE & Residual dimensions & $-0.000 \pm 0.000$ & $0.321 \pm 0.079$ & $0.862 \pm 0.178$ & $\mathbf{0.988 \pm 0.002}$ & $\mathbf{0.988 \pm 0.001}$ \\
kNN & Full dimension & $0.792 \pm 0.003$ & \cellcolor{gray!15} $\mathbf{0.983 \pm 0.000}$ & \cellcolor{gray!15} $\mathbf{0.994 \pm 0.000}$ & \cellcolor{gray!15} $\mathbf{0.995 \pm 0.000}$ & \cellcolor{gray!15} $\mathbf{0.994 \pm 0.000}$ \\
kNN & Residual dimensions & $-0.000 \pm 0.000$ & $0.113 \pm 0.023$ & $0.430 \pm 0.029$ & $0.636 \pm 0.013$ & $\mathbf{0.993 \pm 0.001}$ \\
logdet & Full dimension & $0.795 \pm 0.002$ & \cellcolor{gray!15} $\mathbf{0.997 \pm 0.000}$ & \cellcolor{gray!15} $\mathbf{0.996 \pm 0.000}$ & \cellcolor{gray!15} $\mathbf{0.996 \pm 0.000}$ & \cellcolor{gray!15} $\mathbf{0.993 \pm 0.000}$ \\
logdet & Residual dimensions & $-0.000 \pm 0.000$ & $0.011 \pm 0.001$ & $0.012 \pm 0.002$ & $0.014 \pm 0.005$ & $0.015 \pm 0.005$ \\
Stein & Full dimension & $0.794 \pm 0.003$ & \cellcolor{gray!15} $\mathbf{0.996 \pm 0.000}$ & \cellcolor{gray!15} $\mathbf{0.992 \pm 0.000}$ & \cellcolor{gray!15} $\mathbf{0.958 \pm 0.000}$ & \cellcolor{gray!15} $\mathbf{0.989 \pm 0.000}$ \\
Stein & Residual dimensions & $-0.000 \pm 0.000$ & $0.063 \pm 0.014$ & $0.180 \pm 0.019$ & $0.055 \pm 0.005$ & $0.048 \pm 0.009$ \\
\bottomrule
\end{tabular}
\end{table}

\begin{table}[H]
\caption{Hopper linear position recovery $R^2$ versus latent dimension. 
The evaluated signal dimension is 7, but the true latent system might contain more uncontrolled signal variables. 
Recovery improves with higher dimensional representations (except for KDE), which could be because the system has more latent dimensions.
Another potential explanation is that Hopper dynamics do not fully follow the Assumptions of our Gaussian Corollary \ref{cor:gaussian-readout}, and higher dimensional latents allow more easily to find directions that are linearly related to true latents, even though the encoding is slightly nonlinear.
This hypothesis is corroborated by the residual representation analysis as in Table \ref{tab:randstep-latent-dimension}, showing that signal information is not only contained in a linear subspace with dimension 7.
Naive logdet estimation encountered stability issues for higher dimensions, likely due to latent dimension reaching batchsize.
Gray boxes indicate conditions which are expected to result in successful recovery.
}
\label{tab:hopper-latent-dimension}
\vspace{3mm}
\tiny
\centering
\begin{tabular}{llcccccc}
\toprule
Estimator & Probe & $d_z=5$ & $d_z=7$ & $d_z=16$ & $d_z=32$ & $d_z=64$ & $d_z=128$ \\
\midrule
KDE & Full dim. & $0.62 \pm 0.01$ &  $0.56 \pm 0.06$ & \cellcolor{gray!15} $0.58 \pm 0.01$ & $\cellcolor{gray!15} 0.55 \pm 0.02$ & \cellcolor{gray!15} $0.35 \pm 0.01$ & \cellcolor{gray!15} $0.22 \pm 0.02$ \\
KDE & Residual dims. & $-0.00 \pm 0.00$ & $-0.00 \pm 0.00$ & $0.44 \pm 0.01$ & $0.34 \pm 0.02$ & $0.07 \pm 0.01$ & $0.12 \pm 0.02$ \\
kNN & Full dim. & $0.61 \pm 0.01$ &  $0.65 \pm 0.03$ & \cellcolor{gray!15} $0.77 \pm 0.00$ & \cellcolor{gray!15} $0.84 \pm 0.01$ & \cellcolor{gray!15} $0.89 \pm 0.01$ & \cellcolor{gray!15} $\mathbf{0.91 \pm 0.00}$ \\
kNN & Residual dims. & $-0.00 \pm 0.00$ & $-0.00 \pm 0.00$ & $0.68 \pm 0.01$ & $0.78 \pm 0.01$ & $0.85 \pm 0.01$ & $0.88 \pm 0.00$ \\
logdet & Full dim. & $0.60 \pm 0.03$ & $0.77 \pm 0.01$ & \cellcolor{gray!15} $0.87 \pm 0.01$ & \cellcolor{gray!15} --- &\cellcolor{gray!15} --- & \cellcolor{gray!15} --- \\
logdet & Residual dims. & $-0.00 \pm 0.00$ & $-0.00 \pm 0.00$ & $0.42 \pm 0.02$ & --- & --- & --- \\
Stein & Full dim. & $0.70 \pm 0.04$ &  $0.74 \pm 0.01$ & \cellcolor{gray!15} $0.85 \pm 0.00$ & \cellcolor{gray!15} $0.89 \pm 0.00$ & \cellcolor{gray!15} $\mathbf{0.92 \pm 0.00}$ & \cellcolor{gray!15} $\mathbf{0.94 \pm 0.00}$ \\
Stein & Residual dims. & $-0.00 \pm 0.00$ & $-0.00 \pm 0.00$ & $0.74 \pm 0.01$ & $0.87 \pm 0.00$ & $\mathbf{0.91 \pm 0.00}$ & $\mathbf{0.93 \pm 0.00}$ \\
\bottomrule
\end{tabular}
\end{table}

\subsubsection{Causal3DIdent}

\begin{table}[H]
    \caption{Nonlinear $R^2$ for the Causal3DIdent experiment (Table~\ref{tab:c3d}), computed with trained \ac{MLP} predictor on held-out dataset.}
    \label{tab:c3dnonlinear}
    \vspace{3mm}
    \tiny
    \begin{tabular}{llcccccccccc}
\toprule
Scenario & Estimator & $p_{\text{obj-}x}$ & $p_{\text{obj-}y}$ & $p_{\text{obj-}z}$ & $\theta_{\text{obj-}\alpha}$ & $\theta_{\text{obj-}\beta}$ & $\theta_{\text{obj-}\gamma}$ & $h_\text{obj}$ & $p_\text{spot}$ & $h_\text{spot}$ & $h_\text{back}$ \\
\midrule
Environment signal & KDE & $0.37$ & $0.53$ & $0.36$ & $0.11$ & $0.09$ & $0.16$ & $0.54$ & \cellcolor{gray!15} $\mathbf{0.99}$ & \cellcolor{gray!15} $\mathbf{0.97}$ & \cellcolor{gray!15} $\mathbf{0.99}$ \\
Environment signal & kNN & $0.20$ & $0.47$ & $0.37$ & $0.07$ & $0.02$ & $0.08$ & $0.18$ & \cellcolor{gray!15} $\mathbf{0.98}$ & \cellcolor{gray!15} $\mathbf{0.96}$ & \cellcolor{gray!15} $\mathbf{0.99}$ \\
Environment signal & logdet & $-0.02$ & $-0.03$ & $-0.01$ & $-0.03$ & $-0.02$ & $-0.01$ & $-0.02$ & \cellcolor{gray!15} $\mathbf{0.99}$ & \cellcolor{gray!15} $0.57$ & \cellcolor{gray!15} $\mathbf{1.00}$ \\
Environment signal & Stein & $0.29$ & $0.67$ & $0.38$ & $0.11$ & $0.04$ & $0.04$ & $0.60$ & \cellcolor{gray!15} $\mathbf{0.99}$ & \cellcolor{gray!15} $\mathbf{0.97}$ & \cellcolor{gray!15} $\mathbf{0.99}$ \\
\midrule
Object signal & KDE & \cellcolor{gray!15} $\mathbf{0.98}$ & \cellcolor{gray!15} $\mathbf{0.98}$ & \cellcolor{gray!15} $\mathbf{0.97}$ & \cellcolor{gray!15} $\mathbf{0.93}$ & \cellcolor{gray!15} $\mathbf{0.97}$ & \cellcolor{gray!15} $0.83$ & \cellcolor{gray!15} $\mathbf{0.91}$ & $0.16$ & $0.03$ & $0.36$ \\
Object signal & kNN & \cellcolor{gray!15} $\mathbf{0.98}$ & \cellcolor{gray!15} $\mathbf{0.96}$ & \cellcolor{gray!15} $\mathbf{0.95}$ & \cellcolor{gray!15} $0.85$ & \cellcolor{gray!15} $\mathbf{0.91}$ & \cellcolor{gray!15} $0.26$ & \cellcolor{gray!15} $\mathbf{0.91}$ & $0.14$ & $-0.04$ & $0.67$ \\
Object signal & logdet & \cellcolor{gray!15} $\mathbf{0.99}$ & \cellcolor{gray!15} $\mathbf{0.99}$ & \cellcolor{gray!15} $\mathbf{0.98}$ & \cellcolor{gray!15} $\mathbf{0.95}$ & \cellcolor{gray!15} $\mathbf{0.97}$ & \cellcolor{gray!15} $\mathbf{0.95}$ & \cellcolor{gray!15} $\mathbf{0.92}$ & $0.17$ & $-0.03$ & $-0.05$ \\
Object signal & Stein & \cellcolor{gray!15} $\mathbf{0.99}$ & \cellcolor{gray!15} $\mathbf{0.98}$ & \cellcolor{gray!15} $\mathbf{0.98}$ & \cellcolor{gray!15} $\mathbf{0.96}$ & \cellcolor{gray!15} $\mathbf{0.98}$ & \cellcolor{gray!15} $\mathbf{0.95}$ & \cellcolor{gray!15} $\mathbf{0.92}$ & $0.12$ & $-0.07$ & $0.55$ \\
\bottomrule
\end{tabular}
\end{table}

\newpage
\subsubsection{Hopper}

\begin{figure}[h!]
    \centering
    \includegraphics[width=\linewidth]{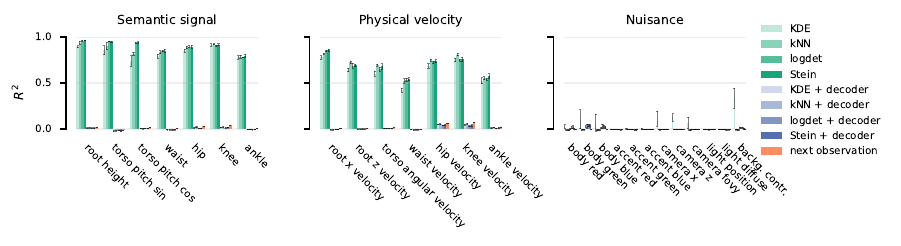}\par
    \vspace{2mm}
    \includegraphics[width=\linewidth]{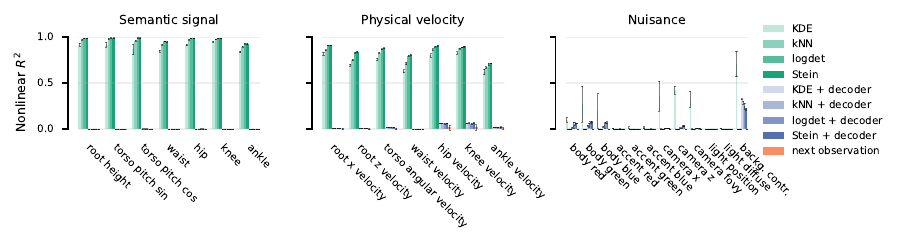}
    \caption{Linear (Top) and nonlinear (Bottom) $R^2$ with physical (signal) and nuisance variables for the hopper experiment. Velocity is read out from \ac{LSTM} state. Background pattern nuisance is not included.}
    \label{fig:hopper_R2_appendix}
\end{figure}

% \begin{figure}[h!]
%     \centering
%     \includegraphics{figs/hopper_appendix.pdf}
%     \caption{Extended example of reconstructed hopper data (cf. Figure~\ref{fig:hopper}B).}
%     \label{fig:hopper_appendix}
% \end{figure}

\end{document}